\documentclass{article}
\usepackage[ruled,vlined]{algorithm2e}
\usepackage{microtype}
\usepackage{graphicx}
\usepackage{subcaption}
\usepackage{booktabs} 

\usepackage{threeparttable}
\usepackage{adjustbox}
\usepackage{booktabs} 
\usepackage{longtable} 
\usepackage{lscape} 
\usepackage{multirow}
\usepackage{placeins} 

\usepackage{hyperref}

\usepackage[preprint]{icml2026}

\usepackage{amsmath}
\usepackage{amssymb}
\usepackage{mathtools}
\usepackage{amsthm}

\usepackage[capitalize,noabbrev]{cleveref}

\theoremstyle{plain}
\newtheorem{theorem}{Theorem}[section]
\newtheorem{proposition}[theorem]{Proposition}
\newtheorem{lemma}[theorem]{Lemma}

\theoremstyle{definition}

\theoremstyle{remark}

\usepackage[textsize=tiny]{todonotes}

\icmltitlerunning{GenIAS: Generator for Instantiating Anomalies in time Series}

\begin{document}

\twocolumn[
  \icmltitle{GenIAS: Generator for Instantiating Anomalies in Time Series}



  \icmlsetsymbol{equal}{*}

  \begin{icmlauthorlist}
    \icmlauthor{Zahra Zamanzadeh Darban}{equal,yyy}
    \icmlauthor{Mike Wang}{comp}
    \icmlauthor{Geoffrey~I. Webb}{yyy}
    \icmlauthor{Shirui Pan}{sch}
    \icmlauthor{Charu C. Aggarwal}{iii}
    \icmlauthor{Mahsa Salehi}{yyy}
  \end{icmlauthorlist}

  \icmlaffiliation{yyy}{Monash University, Melbourne, Victoria, Australia}
  \icmlaffiliation{iii}{IBM T. J. Watson Research Center, Yorktown Heights, NY, USA}
  \icmlaffiliation{sch}{Griffith University, Gold Coast, Queensland, Australia}
  \icmlaffiliation{comp}{The University of Melbourne, Parkville, Victoria, Australia}

  \icmlcorrespondingauthor{Zahra Zamanzadeh Darban}{zahra.zamanzadeh@monash.edu}

  \icmlkeywords{Machine Learning, ICML}

  \vskip 0.3in
]



\printAffiliationsAndNotice{}  


\begin{abstract}
Synthetic anomaly injection is a recent and promising approach for time series anomaly detection (TSAD), but existing methods rely on ad hoc, hand-crafted strategies applied to raw time series that fail to capture diverse and complex anomalous patterns, particularly in multivariate settings. We propose a synthetic anomaly generation method named \underline{Gen}erator for \underline{I}nstantiating \underline{A}nomalies in Time \underline{S}eries ({\bf GenIAS}), which generates realistic and diverse anomalies via a novel learnable perturbation in the latent space of a variational autoencoder. This enables abnormal patterns to be injected across different temporal segments at varying scales based on variational reparameterization. To generate anomalies that align with normal patterns while remaining distinguishable, we introduce a learning strategy that jointly learns the perturbation scale and compact latent representations via a tunable prior, which improves the distinguishability of generated anomalies, as supported by our theoretical analysis.
Extensive experiments show that GenIAS produces more diverse and realistic anomalies, and that detection models trained with these anomalies outperform 17 baseline methods on 9 popular TSAD benchmarks. 

\end{abstract}

\section{Introduction}
Unsupervised TSAD methods often learn the boundary of normal data for detection without labeled anomalies, making them vulnerable to diverse or subtle anomaly patterns. This is particularly true in complex multivariate scenarios, where anomalous patterns can vary significantly across dimensions or involve intricate correlations.

Anomaly injection techniques are gaining popularity for enhancing normal-boundary learning by addressing the lack of labels. Models like CARLA~\cite{DARBAN2025carla}, CutAddPaste~\cite{wang2024cutaddpaste}, NCAD~\cite{ncad2022}, and COUTA~\cite{Calibrated} show the promise of synthetic anomaly injection in improving detection, but each faces notable limitations. For instance, CARLA, CutAddPaste, and COUTA are limited to predefined anomaly types, such as point, seasonal, and shapelet anomalies. CutAddPaste employs a fixed “cut-add-paste” augmentation strategy, which may produce unrealistic patterns of time series. NCAD relies on prior knowledge of anomalies to generate realistic anomalies, but the diversity of its generated anomalies is constrained by the types of anomalies present in the training data.
Moreover, most existing anomaly injection models operate directly in raw time series spaces, without modelling their normal patterns. 
This can result in unrealistic alignment between injected anomalies and the original time series components, producing anomalies that are out of context or entangled with normal patterns, and thus failing to provide anomaly discriminative information for TSAD training.

We address these gaps by introducing GenIAS, a novel generative model designed to produce realistic and diverse synthetic anomalies for time series data. In contrast to existing methods that operate directly on raw time series, GenIAS performs anomaly generation in the latent space using a hybrid Temporal Convolutional Network–Variational Autoencoder (TCN-VAE) architecture. Through variational reparameterization, this design encodes the temporal and inter-variable dependencies of normal patterns in multivariate time series within a continuous latent space, enabling fine-grained and controllable perturbations for generating diverse and realistic anomalies. This is achieved by our novel learnable perturbation mechanism in the learned latent space, which modulates latent variance to induce anomalies with varying degrees of deviation across different time series components.

To produce anomalies that align with the temporal properties of the original time series while remaining distinguishable from normal patterns, we introduce a joint training strategy that learns both the latent distribution of normal time series and the perturbation mechanism by enforcing two constraints. First, we propose a perturbation loss that encourages deviations of reasonable magnitude, avoiding perturbations that are either too trivial or excessively large. Second, we introduce a compact loss with a tunable variance prior initialized to a small value (e.g., $<1$) to induce compact latent representations in VAEs, as illustrated in Fig.~\ref{fig:genias_idea}. Our theoretical analysis shows that such compactness reduces overlap between the generated anomalies and normal patterns when applying the proposed learnable perturbation.

GenIAS injects anomalies through two stages, where a learned perturbation is first applied to the latent variance of normal time series representations to produce anomalous latent samples that are decoded into anomalous windows. To better reflect realistic TSAD scenarios, where anomalies are often localized rather than spanning all dimensions, a deviation-based patching strategy is then applied to selectively inject anomalous components from the anomalous window into the original window on a per-dimension basis, while preserving normal context. To evaluate anomaly generation quality, we introduce two complementary representation-based metrics that assess realism, defined as proximity to real anomalies in representation space, and diversity of the generated anomalies. GenIAS achieves a better balance in anomaly generation quality, and detection models trained with its generated anomalies consistently outperform 17 baselines across nine univariate and multivariate datasets.

Our key contributions are as follows:


\begin{figure}[t]
    \centering
    \includegraphics[width=0.6\linewidth]{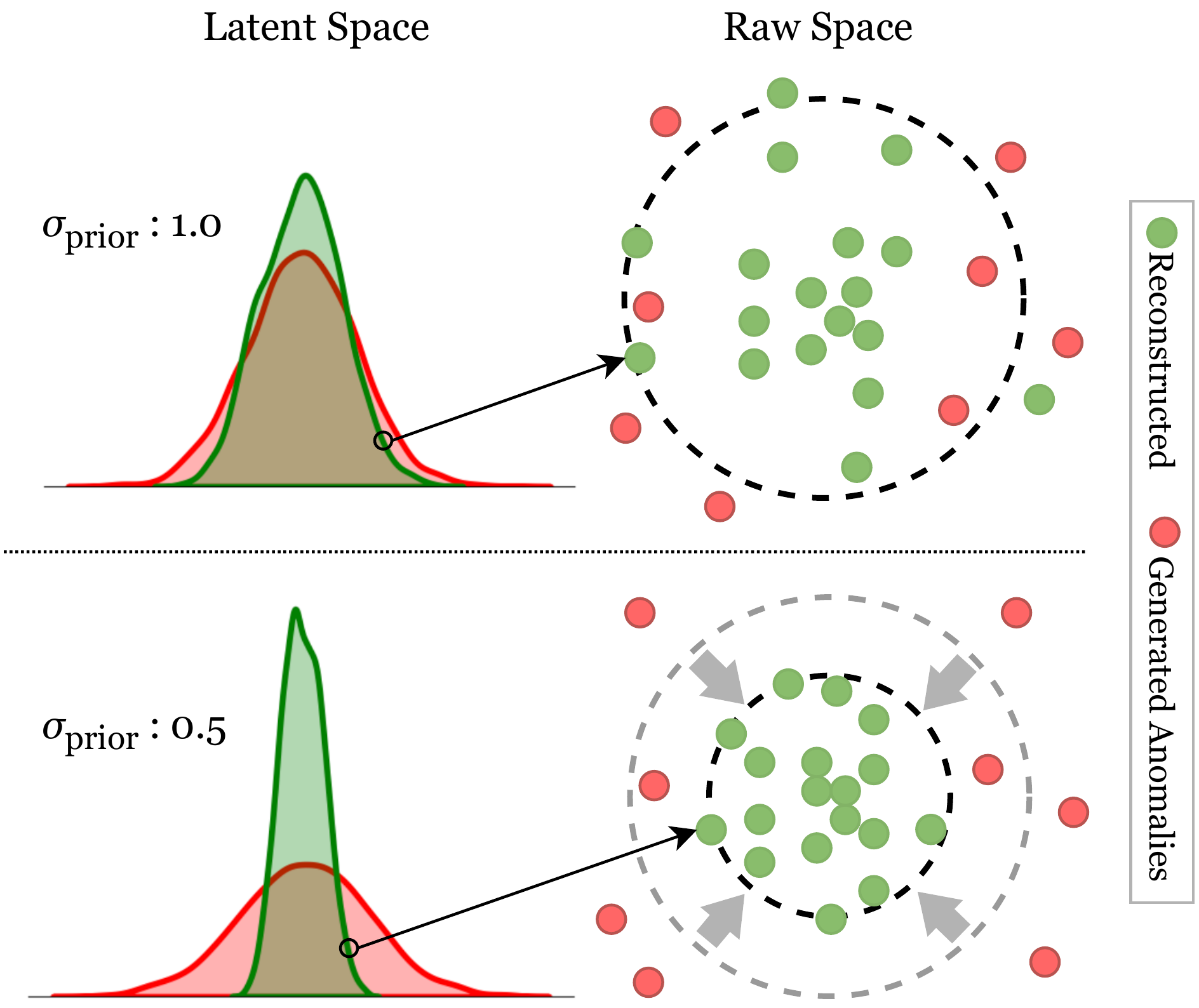}
    \caption{Effect of compactness. A smaller prior $(\sigma_{\text{prior}}=0.5$ vs $\sigma_{\text{prior}}=1.0)$ variance produces compact normal latent representations and improved separation from generated anomalies }
    \label{fig:genias_idea}
    \vspace{-8pt}
\end{figure}

\begin{itemize}
    \item We propose a novel synthetic anomaly generation method, GenIAS, that operates in the latent space of variational autoencoders and generates diverse anomalies with varying degrees of deviation across different time series components, validating anomaly injection via controllable latent space perturbation as a promising strategy for improving unsupervised TSAD.
    \item We introduce an optimization strategy for GenIAS that jointly learns the latent distribution of time series data and the perturbation mechanism. This ensures that injected deviations are informative and distinguishable while remaining coherent with the normal time series.
    \item We propose two metrics for evaluating the diversity and realism for measuring anomaly generation, where GenIAS demonstrates better and more balanced results compared to the synthetic anomaly generation methods. We further show that GenIAS-enabled detection models demonstrate superior performance compared to 17 baseline methods, translating improved anomaly generation into stronger TSAD performance.
\end{itemize} 

\vspace{-20pt}

\section{Related Work}
The field of TSAD has advanced significantly, spanning traditional statistical methods to SOTA deep learning approaches~\cite{schmidl2022anomaly,audibert2022deep,darban2024deep}. Early efforts included models like One-Class SVM (OC-SVM)~\cite{scholkopf1999support}, LOF~\cite{breunig2000lof}, and Isolation Forest~\cite{liu2008isolation}, which struggled with handling high-dimensional and sequential data. In contrast, deep learning models like Donut~\cite{xu2018donut}, LSTM-VAE~\cite{park2018lstmvae} and OmniAnomaly~\cite{su2019robust} introduced temporal dependency modeling, advancing TSAD performance.

\sloppy
Representation-based methods have pushed TSAD forward by focusing on meaningful embeddings. TS2Vec~\cite{yue2022ts2vec} applies self-supervised learning to capture multi-level semantic features, while DCdetector~\cite{yang2023dcdetector} employs dual attention for permutation-invariant representations. Recent models such as TimesNet~\cite{wu2023timesnet} and TranAD~\cite{tuli2022tranad} leverage attention mechanisms to capture complex dependencies. Methods like MTAD-GAT~\cite{zhao2020mtad} utilize graph attention networks to model multivariate dependencies, and THOC~\cite{shen2020thoc} incorporates hierarchical modelling to capture temporal patterns across multiple scales. Anomaly Transformer~\cite{xu2021anomalytran} leverages anomaly-sensitive features through attention mechanisms.

OE~\cite{hendrycks2018OE} improves robustness by introducing out-of-distribution examples during training. NCAD~\cite{ncad2022} extends this idea to time series by injecting mixed contextual and random point anomalies to better separate normal and anomalous patterns.
Anomaly injection and perturbation methods have become key strategies in TSAD, with methods like CutAddPaste~\cite{wang2024cutaddpaste} and COUTA~\cite{Calibrated} simulating diverse anomaly patterns during training. These techniques are also used in CARLA~\cite{DARBAN2025carla} to create negative pairs for contrastive learning, enhancing generalization to unseen anomalies.

\begin{figure*}[t]
\centering
\includegraphics[width=1\linewidth]{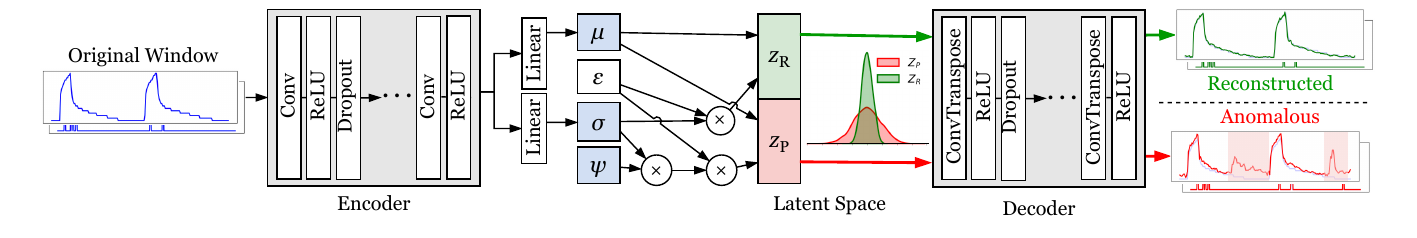}
    \caption{Overall architecture of GenIAS: The model comprises an encoder, a structured latent space, and a decoder. The encoder, based on a TCN-VAE, maps input time series windows into a latent representation by learning temporal dependencies. Latent space regularization encourages compact normal representations, and anomalies are generated by perturbing the variance of the learned distribution. The decoder reconstructs both normal and perturbed windows using transpose convolutional layers.}
    \label{fig:genias_arc}
        \vspace{-10pt}
\end{figure*}

\section{GenIAS Method}
\noindent\textbf{Problem Definition.}
Let $\mathcal{D}$ be a time series dataset consisting of $N$ windows, where each window $\mathbf{X} \in \mathbb{R}^{T \times D}$ has length $T$ and $D$ variables, drawn from an unknown data distribution $p(\mathbf{X})$.
In typical TSAD settings, $\mathcal{D}$ is dominated by normal samples, while anomalies are rare. The objective of  anomaly generation is to learn a mapping
$\Gamma: \mathbf{X} \rightarrow \widetilde{\mathbf{X}}$
that transforms a normal window $\mathbf{X}$ into an anomalous window
$\widetilde{\mathbf{X}} \in \mathbb{R}^{T \times D}$.
The generated anomalies should preserve the local temporal context of $\mathbf{X}$ while introducing controlled deviations that make $\widetilde{\mathbf{X}}$ distinguishable from normal patterns and informative for training TSAD models.


For anomaly generation via latent space perturbation, we identify two key considerations for generating anomaly-discriminative samples. The first is to design a latent space perturbation mechanism for \(\mathbf{z} \in \mathbb{R}^{L}\), parameterized by \(p(\mathbf{z} \mid \mathbf{X})\), that injects controlled deviations into the latent representation, which serves as the input to the decoder for construct anomalous windows, to generate anomalies that satisfy the above properties. The second is to densify normal representations in latent space, which reduces overlap between normal samples and generated anomalies. This is applicable to both univariate and multivariate settings. 


\medskip

\subsection{Anomaly Generation using GenIAS}\label{sec:anom_generation}

\noindent\textbf{GenIAS Overview.}
The core of GenIAS is a novel learnable perturbation mechanism operating on a variational latent space that is explicitly learned through a tailored training strategy, where distributions of normal time series data and the perturbation mechanism are jointly optimized. This joint optimization is motivated by the above design considerations, ensuring that the learned perturbations align well with the context of the normal time series distribution while being informative for anomaly generation.

\noindent\textbf{Generation Workflow.} As shown in Figure \ref{fig:genias_arc}, GenIAS adopts an encoder-decoder architecture, where a Temporal Convolutional Network (TCN) is used as the backbone to extract complex temporal patterns and dependencies through stacked convolutional layers interleaved with nonlinear activations and dropout. A variational autoencoder is then used to model a continuous latent Gaussian distribution $\mathcal{N}(\mu, \sigma^2)$, where $\mu$ and $\sigma^2$ are $L$-dimensional vectors capturing the central tendency (mean) and spread (variance) of normal patterns, respectively, providing a basis for controlled anomaly generation. A normal window can be generated by first sampling a latent vector $\mathbf{z}=\mu + (\sigma \odot \epsilon)$, which is then passed to the decoder to reconstruct the corresponding time series window.

\noindent\textbf{Anomaly Injection Pipeline.}\label{sec:anom_gen}
GenIAS employs a two-stage pipeline for anomaly generation, consisting of anomalous window generation and deviation-based patching. Learnable perturbations are applied during latent sampling of a trained VAE to generate anomalous latent vectors (Eq.\ref{eq:z_perturbed}), which are decoded into anomalous windows (detailed in Algo.~\ref{alg:genias}).

While the reconstructed anomalous window captures coherent abnormal patterns, directly replacing the full window can be insufficient or overly aggressive for TSAD, as real anomalies are often localized. We therefore apply a deviation-based patching strategy that selectively replaces entire dimensions from the normal trajectory with their anomalous counterparts only when a non-trivial deviation is detected.
The patching mechanism is detailed in Sec.\ref{sec:patching}.

\noindent\textbf{Deviation via Variance Perturbation.} 
Although VAEs are commonly used for modelling and generating normal data, we exploit their variational latent distributions to generate anomalies by adjusting the variance around the latent mean to produce desired deviations. We propose perturbing the variance of the learned VAE instead of the mean so that the generated anomalies follow the temporal property and central tendency of the times series. To generate anomalous representations $ \tilde{\mathbf{z}} $, the standard deviation $\sigma$ is perturbed while keeping $\mu$ fixed:
{\footnotesize
\begin{equation}
\vspace{-3pt}
\tilde{\mathbf{z}} = \mu + \psi \cdot (\sigma \odot \epsilon),
\label{eq:z_perturbed}
\vspace{-2pt}
\end{equation}}
where $ \psi $ is a learned perturbation scale, and $ \epsilon \sim \mathcal{N}(\mathbf{0}, \mathbf{I}) $ is a noise vector. We find that perturbing the latent mean can introduce large deviations, but it violates the central tendency, often exhibiting out of context temporal properties that are less informative for representing realistic anomalies.
\subsection{GenIAS's Joint Learning}
We jointly learn the perturbation scale and the time series latent distribution to ensure compatibility. In addition to the reconstruction loss for encoding the context of time series (Eq. \ref{eq:reconstruction_loss}), we introduce two constraints to tighten the variance of the normal latent distribution and to regularize an effective perturbation scale. Specifically, we enforce a compact normal manifold through a compact loss (Eq. \ref{eq:kl_divergence_loss}) to reduce overlap between generated anomalies and normal data, and calibrate the perturbation scale to produce context-aligned deviations using a perturbation loss (Eq. \ref{eq:perturbation_loss}). Intuitively, effective synthetic anomalies should deviate from normal patterns without leaving the local context, since overly subtle deviations label near normal samples as anomalies and introduce label noise, while overly large deviations often become out of context and uninformative.

\noindent\textbf{Compact KL Divergence for Distinguishability.} The standard KL divergence loss is mainly used to regularize the latent space by aligning the approximate posterior distribution $ q(z|\mathbf{X}) $, parameterized by  $\mu$ and $\sigma^2$, with a standard Gaussian prior $p(z)$, defined as $\mathcal{N}(0, 1)$. However, this does not explicitly control the spread of the normal latent distribution, which can become overly diffuse and increase overlap between generated anomalies and normal samples. We address this by introducing a compact KL divergence loss with a tunable prior variance $ \sigma_{\text{prior}}^2 $:
{\footnotesize
\begin{equation}
\mathcal{L}_{\text{comp-KL}} = -\frac{1}{2} \sum_{i=1}^{|\mathbf{\mathcal{D}}|} \sum_{j=1}^{L} \Big[ 1 + \log(\sigma^2_{ij}) - \mu_{ij}^2 - \frac{\sigma^2_{ij}}{\sigma_{\text{prior}}^2} + 2 \log(\sigma_{\text{prior}}) \Big],
\label{eq:kl_divergence_loss}
\end{equation}
}
Where $\mu_{ij}$ and $\sigma^2_{ij}$ are the mean and variance of the latent distribution for the $j$-th dimension of the $i$-th sample, and $\sigma_{\text{prior}} < 1$ is the prior standard deviation.
We show in Lemma \ref{lemma:compact} that introducing the tunable prior variance $\sigma_{\text{prior}}$ leads to tighter latent representations. 
\begin{lemma}[Compact Latent Space in Compact KL Loss] \label{lemma:compact}
Consider the KL divergence between the estimated posterior and the prior distributions for a latent dimension $j$. If the compact KL loss (Eq. ~\ref{eq:kl_divergence_loss}) is minimized (or its gradient dominates locally) with a prior variance $\sigma^2_{\text{prior}} < 1$, then for each latent variable $z_j$ the optimal posterior variance satisfies $\sigma^2_j = \sigma^2_{\text{prior}}$. Consequently, the posterior variance of each latent variable is driven toward the prior variance $\sigma^2_{\text{prior}}$, thereby imposing a compact spread in the latent space.
\end{lemma}

The full proof is provided in \ref{lemma:compact_full}. We now show in the following proposition that limiting the latent variance improves the distinguishability between generated anomalous and normal latent distributions under variance scaling perturbations.

\begin{proposition}[Compactness leads to greater separation in the latent space] \label{theorem:1}
Let a VAE be trained on normal data with a latent prior \( \mathcal{N}(0, \sigma^2_{\text{prior}}) \), where \( \sigma_{\text{prior}} < 1 \). The encoder posterior for normal samples follows \( \mathcal{N}(\mu, \sigma^2_{\text{normal}}) \), with compactness enforced by KL regularization. If the encoder \( \phi \) applies a perturbation scale \( \psi \), inflating variance as  
\(
\sigma^2_{\text{anom}} = \psi \sigma^2_{\text{normal}}, \quad \text{with } \psi > 1,
\)
then the KL divergence between normal and anomalous latent distributions strictly increases compared to the case where \( \sigma_{\text{prior}} \geq 1 \).
\end{proposition}
The proof is provided in \ref{app:proof_t_1}. The increased KL divergence under lower variance implies statistically improved distinguishability. We also illustrate in Fig.~\ref{fig:mse_hist} in the Appendix how our compact KL loss leads to the separation of reconstruction error in the raw space.

\noindent\textbf{Perturbation Loss for Informative Deviation.}\label{sec:pert_loss}
We define a perturbation loss to control the magnitude of deviation introduced by latent perturbations composed of two complementary terms. 

The first term is a {\it triplet-style margin loss} that enforces a minimum separation between the original sample $\mathbf{X}$ and the generated anomalous sample $\widetilde{\mathbf{X}}$, relative to the reconstructed normal sample $\hat{\mathbf{X}}$. This term ensures that generated anomalies deviate sufficiently from the original time series by pushing anomalous windows away from normal windows in the raw time series space with a minimum margin \(\delta_{\text{min}}\).


The second term acts as a {\it regularizer} that prevents the anomalous windows \( \tilde{\mathbf{X}} \) from deviating excessively from their original samples, thereby preserving realism. It penalizes deviations where distance \( d(\mathbf{X}, \tilde{\mathbf{X}}) \) exceeds a maximum allowable threshold \(\delta_{\text{max}}\), ensuring perturbations remain within a plausible range.
The perturbation loss is formulated as:
{\footnotesize
\begin{align}
\mathcal{L}_{\text{perturb}} =\; & \frac{1}{|\mathcal{D}|} \sum_{\mathbf{X} \in \mathcal{D}} \Big[ \max\big( d(\mathbf{X}, \hat{\mathbf{X}}) - d(\mathbf{X}, \tilde{\mathbf{X}}) + \delta_{\text{min}}, 0 \big) \Big] \nonumber \\
& + \frac{1}{|\mathcal{D}|} \sum_{\mathbf{X} \in \mathcal{D}} \max\big( d(\mathbf{X}, \tilde{\mathbf{X}}) - \delta_{\text{max}}, 0 \big)
\label{eq:perturbation_loss}
\end{align}
}
These terms constrain the perturbation magnitude to avoid overly obvious deviations while ensuring anomalies remain distinguishable.

\noindent\textbf{Zero-Perturbation Loss for Zero-valued 
Dimensions.}\label{sec:zero_pert_loss} 
Time series windows may contain dimensions that are entirely zero, corresponding to inactive or sparse signals. Directly applying perturbations to such dimensions can result in unrealistic anomalies. We introduce a zero-perturbation loss that regularizes deviations in zero-valued dimensions. Specifically, the loss penalizes large deviations in dimensions where the original signal is zero, encouraging perturbations to remain small and plausible while preserving the structural characteristics of these dimensions. The loss is defined as:
{\footnotesize
\begin{equation}
\mathcal{L}_{\text{zero-perturb}} = \frac{1}{|\mathcal{D}|} \sum_{\mathbf{X} \in \mathcal{D}}\big(\frac{1}{T \cdot |\mathcal{Z}|} \sum_{t=1}^{T} \sum_{d \in \mathcal{Z}} \big( d(\mathbf{X}, \tilde{\mathbf{X}}) + 1)^{-1}\big),
\label{eq:zero_perturbation_loss}
\end{equation}
}
where $\mathcal{Z}$ denotes the set of zero-valued dimensions in $\mathbf{X}$. This term discourages unrealistic perturbations in zero-valued channels while allowing controlled deviations when appropriate.

\noindent\textbf{Reconstruction Loss for Learning Normal.}
\label{sec:rec_loss}
The reconstruction loss  is widely used for learning latent representation and is defined as Mean Squared Error (MSE) between an input sample \( \mathbf{X} \) and its reconstruction \( \hat{\mathbf{X}} \):
{\footnotesize
\begin{equation}
\mathcal{L}_{\text{recon}} = \frac{1}{|\mathcal{D}|} \sum_{\mathbf{X} \in \mathcal{D}} \frac{1}{T \cdot D} \sum_{t=1}^{T} \sum_{d=1}^{D} \big( \mathbf{X}[t, d] - \hat{\mathbf{X}}[t, d] \big)^2,
\label{eq:reconstruction_loss}
\end{equation}
}
which is included to capture the temporal properties of normal samples as the basis for perturbation.

\noindent\textbf{Overall Objective.}
\label{sec:total_loss}
The overall objective for GenIAS joint learning can be defined as:
{\footnotesize
\begin{equation}
\mathcal{L}_{\text{total}} = \alpha \cdot \mathcal{L}_{\text{recon}} 
+ \beta \cdot \mathcal{L}_{\text{perturb}} 
+ \gamma \cdot \mathcal{L}_{\text{zero-perturb}} 
+ \zeta \cdot \mathcal{L}_{\text{comp-KL}},
\label{eq:total_loss}
\end{equation}}
where \( \alpha \), \( \beta \), \( \gamma \), and \( \zeta \) are the coefficients for the reconstruction, perturbation, zero-perturbation, and KL divergence losses, respectively.

\begin{figure}[t]
    \centering
    \begin{subfigure}[b]{0.22\textwidth}
        \centering
        \includegraphics[width=\textwidth]{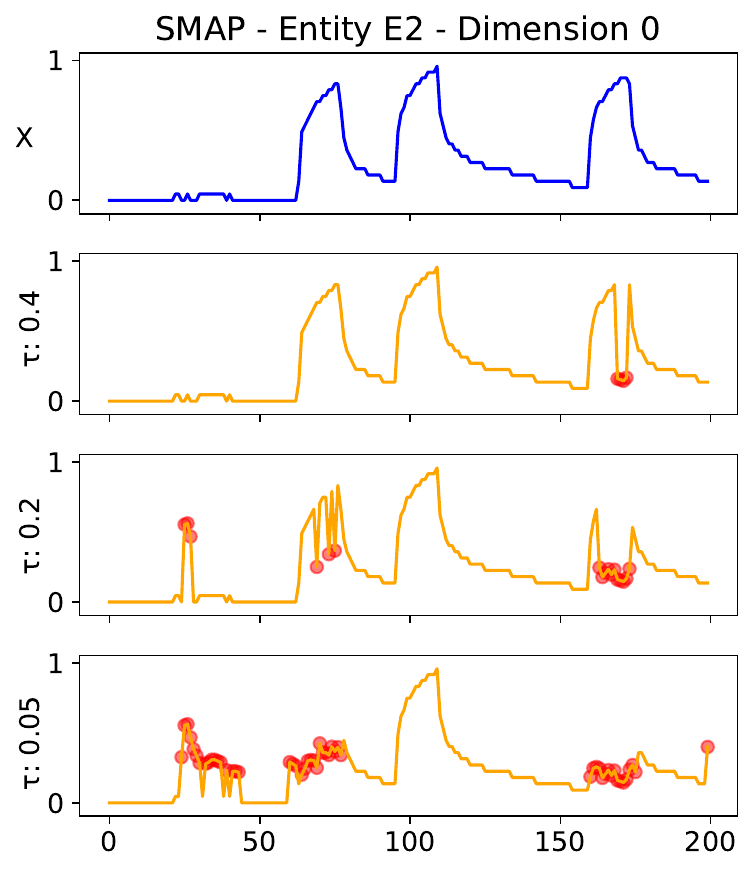}
        \label{fig:smap-e2}
    \end{subfigure}
    \begin{subfigure}{0.212\textwidth}
    \centering
        \includegraphics[width=\textwidth]{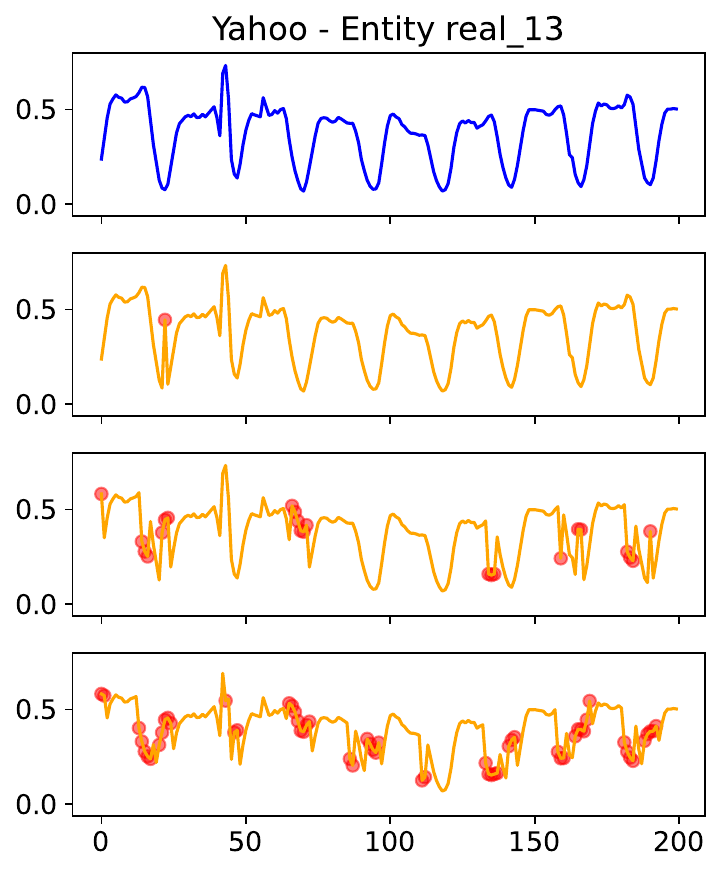}
        \label{fig:yahoo-r13}
    \end{subfigure}
    \caption{Effect of deviation-based patching thresholds on generated anomalies across SMAP and Yahoo datasets.}
    \label{fig:patching}
\end{figure}

\subsection{Deviation Patching Approach} \label{sec:patching}
As discussed in Sec.\ref{sec:anom_gen}, directly using reconstructed anomalous windows has practical limitations. We therefore introduce a deviation-based patching strategy to selectively inject anomalous components while preserving normal context. Specifically, patching is performed independently for each dimension $d$ according to
$\mathbf{X}_{\text{patched},d} = \widetilde{\mathbf{X}}_d$ if $\lVert \mathbf{X}_d - \widetilde{\mathbf{X}}_d \rVert^2 > \tau \cdot (\max(\mathbf{X}_d) - \min(\mathbf{X}_d))$, otherwise $\mathbf{X}_d$, where $\tau$ is a scaling factor controlling the sensitivity of anomaly injection. This strategy replaces entire variable trajectories only when a non-trivial deviation is detected, preserving normal patterns in unaffected dimensions. We show its effect in Fig.~\ref{fig:patching}. The complete procedure is summarized in Algo \ref{alg:patching}.  


\section{Experiments}
We extensively evaluate GenIAS using a diverse set of datasets and baselines in two parts: (1) Generation Evaluation, comparing GenIAS with other anomaly injection models to assess its capability to produce realistic and diverse anomalies; and (2) Case Study – TSAD Evaluation, benchmarking GenIAS against seventeen competing methods using nine real-world MTS and UTS datasets. To assess its impact on TSAD, we replace CARLA’s anomaly generator with GenIAS while keeping the rest of the architecture unchanged. CARLA is selected as the base detector for its SOTA performance, enabling fair comparison with the strongest injection-based method. This comprehensive evaluation allows for a detailed analysis of GenIAS's effectiveness in both anomaly generation and TSAD.

\begin{table}[t]
\centering
\caption{An overview of dataset statistics.}
\label{tab:dset}
\begin{adjustbox}{width=\columnwidth}
\begin{tabular}{c|cccccc}
\hline
\textbf{Benchmark} & \textbf{\# datasets} & \textbf{\# dims} & \textbf{Train size} & \textbf{Test size} & \textbf{AR\%} \\ \hline
MSL~\citep{hundman2018detecting} & 27 & 55 & 58,317  & 73,729  & 10.72\% \\
SMAP~\citep{hundman2018detecting} & 55 & 25 & 140,825 & 444,035 & 13.13\% \\
SMD~\citep{su2019robust} & 28 & 38 & 708,405 & 708,420 & 4.16\%  \\
SWaT~\citep{mathur2016swat} & 1 & 51 & 496,800 & 449,919 & 12.33\%  \\
GECCO~\citep{moritz2018gecco} & 1 & 9 & 69,260 & 69,261 & 1.1\%  \\
SWAN~\citep{DVN/EBCFKM_2020} & 1 & 38 & 60,000 & 60,000 & 32.6\%  \\
UCR~\citep{wu2021current} & 250 & 1  & 2,238,349 & 6,143,541 & 0.6\%  \\
Yahoo-A1~\citep{yahoods} & 67 & 1  & 46,667  & 46,717  & 1.76\%  \\
KPI~\citep{aiops_challenge_2018} & 29 & 1  & 1,048,576 & 2,918,847 & 1.87\%  \\ \hline
\end{tabular}%
\end{adjustbox}
\vspace{-10pt}
\end{table}
\subsection{Benchmark Datasets} \label{sec:datasets}
GenIAS is thoroughly evaluated on nine widely used real-world TSAD benchmark datasets, with key statistics reported in Table~\ref{tab:dset}. The selection includes six MTS and three UTS datasets, covering diverse anomaly types and ratios to ensure a comprehensive comparison against baselines. Detailed dataset descriptions are provided in Appendix~\ref{app:dataset}.

\subsection{Baselines} \label{sec:baselines}
GenIAS is compared with seventeen competitive baseline methods that leverage a wide range of anomaly detection models, which can be categorized into four broad categories: classic anomaly detection methods, positive-unlabeled (P-U) methods, unsupervised deep TSAD methods that do not leverage anomaly generation (US-WI), and unsupervised TSAD methods that employ anomaly generation (US-I).

Specifically, three classic general AD methods--- OCSVM~\cite{scholkopf1999support}, LOF~\cite{breunig2000lof}, and Isolation Forest~\cite{liu2008isolation}---to demonstrate the advantages of deep learning TSAD approaches. LSTM-VAE~\cite{park2018lstmvae} is included to assess the impact of incorporating labeled data in TSAD.
US-WI includes a range of unsupervised TSAD baselines that do not use anomaly generation. Methods such as Donut~\cite{xu2018donut}, OmniAnomaly~\cite{su2019robust}, AnomalyTransformer~\cite{xu2021anomalytran}, and TranAD~\cite{Tuli2022TranADDT} rely on reconstruction error for anomaly detection. THOC~\cite{shen2020timeseries} and TimesNet~\cite{wu2023timesnet} approach TSAD as a forecasting task, while MTAD-GAT~\cite{zhao2020mtad} incorporates graph attention to model temporal and multivariate dependencies. Representation learning methods like TS2Vec~\cite{yue2022ts2vec} and DCdetector~\cite{yang2023dcdetector} aim to learn robust, anomaly-discriminative embeddings.
US-I, the most relevant category for GenIAS, includes unsupervised TSAD methods that use anomaly generation. We compare GenIAS with four recent models in this category: NCAD~\cite{ncad2022}, CutAddPaste~\cite{wang2024cutaddpaste}, COUTA~\cite{Calibrated}, and CARLA~\cite{DARBAN2025carla}.

\subsection{Evaluation Metrics}
\label{sec:metrics}
We adopt seven metrics to evaluate GenIAS's performance in terms of generation quality and TSAD performance.

\noindent \textbf{Generation quality.} 
To quantitatively examine anomaly generation quality, we introduce two key metrics: \textit{Anomalous Representation Proximity (ARP)}, which assesses realism, and \textit{Entropy-Based Diversity Index (EDI)}, which characterizes the diversity, motivated by the idea in AnomalyDiffusion~\cite{hu2024anomalydiffusion}. Due to space constraints, definition and further discussion are in Appendix~\ref{app:eval_metrics}. In general, ARP is the primary metric, reflecting how closely generated anomalies resemble actual ones, while EDI is complementary to ARP by assessing diversity only when realism is satisfied.

\noindent\textbf{TSAD evaluation.} We consider five widely-used, complementary metrics to evaluate TSAD performance: F1 score, AUPR (Area Under the Precision-Recall Curve), Affiliation F1 ~\cite{huet2022local}, PATE ~\cite{Pate2024}, and AUROC (Area Under the Receiver Operating Characteristic Curve). A combination of metrics offers a comprehensive, fair, and robust assessment. See Appendix~\ref{app:tsadmetrics} for details.

\subsection{Implementation Details}
Experiments are run on a server with an NVIDIA A40 GPU, 13 CPU cores, and 250GB RAM. We adopt the unsupervised TSAD architecture and training scheme from CARLA as the default model for GenIAS-enhanced detection. GenIAS is implemented in Python using PyTorch and trained using the Adam optimizer with a learning rate of $10^{-4}$ for 1000 epochs. The default window size $T$ is set to 200, and the latent dimension size $L$ is set to 100 and 50, for MTS and UTS, respectively. The hyperparameters for the total loss are $\alpha=1.0$, $\beta=0.1$, $\gamma=0.0/0.01$ (for UTS/MTS), and $\zeta=0.1$, which are fixed for all datasets. We use baseline results from their original implementations with default hyperparameters. Detailed implementation and hyperparameter settings are provided in Appendix~\ref{sec:implementation}.

\subsection{Generation Performance}
\begin{table}[t]
\small
\caption{Evaluation results of anomaly generation quality for GenIAS and US-I baselines on the ARP metric, with the best in bold and second-best underlined (CAP: CutAddPaste).}
\centering
\begin{adjustbox}{width=0.9\columnwidth}
\begin{tabular}{ccccccc}
\toprule
Metric                & Dataset & COUTA               & NCAD            & CAP             & CARLA           & GenIAS          \\ \midrule
\multirow{10}{*}{ARP} & MSL     & \textbf{0.3613}              & 0.3568          & 0.3406          & 0.3586          & \underline{ 0.3605}    \\
                      & SMAP    & 0.7573              & 0.8052          & \underline{ 0.8110}    & 0.7636          & \textbf{0.8113} \\
                      & SMD     & \underline{ 0.0951}        & 0.0810          & 0.0896          & \textbf{0.0958} & 0.0922          \\
                      & SWaT    & 0.9970              & 0.9964          & 0.9963          & \textbf{0.9980} & \underline{ 0.9974}    \\
                      & GECCO   & \underline{ 0.1296}        & 0.1145          & \textbf{0.1335} & 0.1163          & 0.1265          \\
                      & SWAN    & 0.2400              & 0.2362          & 0.2313          & \underline{ 0.2410}    & \textbf{0.2415} \\
                      & Yahoo   & 0.6541              & 0.6139          & 0.6476          & \underline{ 0.6632}    & \textbf{0.6756} \\
                      & KPI     & 0.6856              & \textbf{0.8427} & 0.7207          & 0.7582          & \underline{ 0.7907}    \\
                      & UCR     & 0.1998              & 0.2059          & 0.1936          & \underline{ 0.2086}    & \textbf{0.2161} \\ \cmidrule(lr){2-7}
                      & Rank    & 3.11                & 3.78            & 3.89            & \underline{ 2.44}      & \textbf{1.78}  \\  \cmidrule(lr){1-7}
\multirow{10}{*}{EDI} & MSL     & \underline{ 0.8251}        & 0.7311          & 0.7169          & 0.5822          & \textbf{0.8833} \\
                      & SMAP    & \underline{ 0.8557}        & 0.6699          & 0.6364          & 0.6335          & \textbf{0.8932} \\
                      & SMD     & \underline{ 0.8782}        & 0.7648          & 0.6404          & 0.8041          & \textbf{0.9262} \\
                      & SWaT    & 0.6425              & 0.6164          & \textbf{0.7967}          & 0.6005          & \underline{ 0.7259}    \\
                      & GECCO   & 0.7379              & 0.5571          & 0.7072          & \textbf{0.9327}          & \underline{ 0.8303}    \\
                      & SWAN    & \underline{ 0.9589}        & 0.9189          & 0.7428          & 0.9286          & \textbf{0.9840} \\
                      & Yahoo   & \underline{ 0.8627}        & 0.7202          & 0.7577          & 0.5948          & \textbf{0.9182} \\
                      & KPI     & \underline{ 0.7607}        & 0.6166          & 0.7505          & 0.6962          & \textbf{0.8442} \\
                      & UCR     & \underline{ 0.8166}        & 0.7321          & 0.7536          & 0.7483          & \textbf{0.8961} \\ \cmidrule(lr){2-7}
                      & Rank    & \underline{ 2.22} & 4.11            & 3.56            & 3.89            & \textbf{1.22}   \\  
                      \bottomrule
\end{tabular}
\end{adjustbox}
\label{table:arp}
\vspace{-0.5cm}
\end{table}
We first examine the anomaly generation quality of GenIAS and the competing methods that employ anomaly injection (i.e., the US-I category, see section \ref{sec:baselines}) in terms of ARP and EDI (Section \ref{sec:metrics}). We report the results of ARP and EDI in Table \ref{table:arp}. GenIAS achieves the highest average ARP ranking, suggesting it generates, overall, the most realistic anomalies across all datasets. The realism of the anomaly distribution is essential for providing anomaly-informed supervision in the absence of labeled anomalies. Similarly, for EDI, GenIAS also achieves the highest average ranking, indicating good diversity in its generated anomalies while sufficiently illustrating the actual anomalies. This diversity implies that the generated anomalies can also lie beyond the original anomaly distribution, which can be beneficial for learning a stricter normal class boundary, further improving TSAD performance, provided that they do not interfere with the identification of the normal region. To achieve this, GenIAS treats creating informative but not excessive diversity as a key consideration through its novel perturbation loss.

\subsection{Visualization} \label{sec:visual}
\begin{figure*}[t]
\centering
    \includegraphics[width=0.90\linewidth]{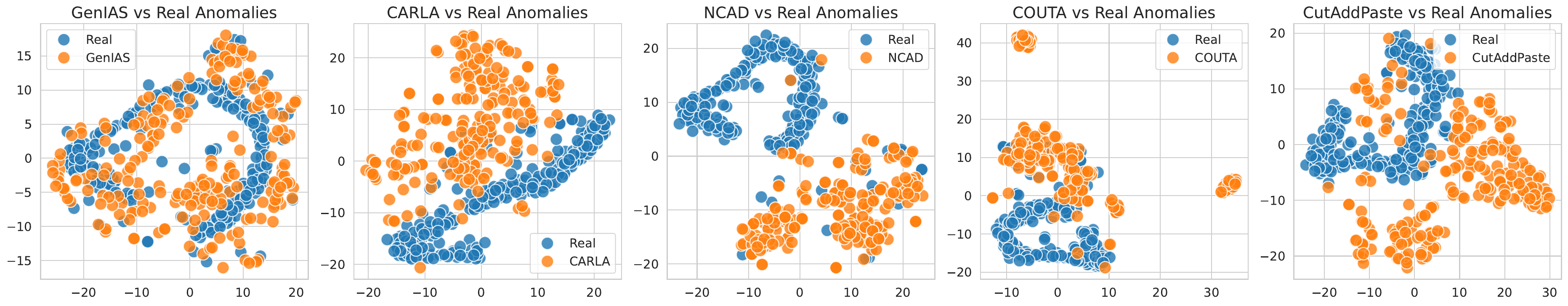}
    \caption{Time series injection methods on the Yahoo dataset. Real anomalies are shown in blue, and synthetic anomalies in orange.}
    \label{fig:yahoo_tsne}
    \label{fig:tsne}
\end{figure*}

\begin{table*}[t]
\vspace{-3pt}
\caption{TSAD performance of GenIAS and baseline methods using F1 score, with the best model in bold and the second-best model underlined. U and M denote results unavailable for models limited to UTS or MTS, respectively.}
\label{table:main}
\centering
\begin{adjustbox}{width=0.75\textwidth}
\begin{tabular}{cccccccccccccc} \toprule
                       &                    & \multicolumn{7}{c}{MTS}                      & \multicolumn{4}{c}{UTS} &       \\ \cmidrule(lr){3-9}  \cmidrule(lr){10-13}
Cat.                  & Methods            & MSL   & SMAP  & SMD   & SWAT  & GECCO & SWAN  & Avg-M & Yahoo  & KPI   & UCR   & Avg-U & Avg   \\ \midrule
\multirow{3}{*}{Classic} & OCSVM (1999)       & 0.308 & 0.411 & 0.319 & 0.743 & 0.297 & 0.557 & 0.439 & 0.608  & 0.148 & 0.011 & 0.256 & 0.378 \\
                       & LOF (2000)         & 0.368 & 0.400 & 0.204 & 0.331 & 0.075 & 0.655 & 0.339 & 0.586  & 0.101 & 0.010 & 0.232 & 0.303 \\
                       & IForest (2008)     & 0.278 & 0.353 & 0.306 & 0.745 & 0.255 & 0.717 & 0.442 & 0.553  & 0.236 & 0.010 & 0.266 & 0.384 \\ \cmidrule(lr){2-14}
PU                    & LSTM-VAE (2018)    & 0.407 & 0.437 & 0.298 & 0.724 & 0.208 & 0.625 & 0.450 & M      & M     & M     & M     & M     \\ \cmidrule(lr){2-14}
\multirow{9}{*}{US-WI}  & Donut (2018)       & U     & U     & U     & U     & U     & U     & U     & 0.489  & 0.126 & 0.009 & 0.208 & U     \\
                       & Omni (2019)        & 0.243 & 0.325 & 0.459 & \textbf{0.763} & 0.166 & 0.541 & 0.416 & M      & M     & M     & M     & M     \\
                       & THOC (2020)        & 0.309 & 0.327 & 0.168 & 0.638 & 0.146 & 0.506 & 0.349 & 0.253  & 0.233 & 0.010 & 0.165 & 0.288 \\
                       & MTAD-GAT (2020)     & 0.473 & 0.519 & 0.347 & 0.242 & 0.313 & 0.599 & 0.416 & M      & M     & M     & M     & M     \\
                       & AnomalyTran (2021) & 0.345 & 0.407 & 0.304 & 0.738 & 0.155 & 0.524 & 0.412 & M      & M     & M     & M     & M     \\
                       & TranAD (2022)      & 0.428 & 0.472 & 0.361 & 0.310 & 0.280 & 0.527 & 0.396 & 0.566  & 0.287 & 0.010 & 0.288 & 0.360 \\
                       & TS2Vec (2022)      & 0.299 & 0.371 & 0.173 & 0.261 & 0.063 & 0.502 & 0.278 & 0.484  & 0.204 & 0.007 & 0.232 & 0.263 \\
                       & TimesNet (2023)    & 0.358 & 0.402 & 0.338 & 0.217 & 0.297 & 0.492 & 0.351 & 0.513  & 0.241 & 0.011 & 0.255 & 0.319 \\
                       & Dcdetector (2023)  & 0.227 & 0.275 & 0.083 & 0.217 & 0.021 & 0.492 & 0.219 & 0.112  & 0.042 & 0.008 & 0.054 & 0.164 \\  \cmidrule(lr){2-14}
\multirow{4}{*}{US-I}   & NCAD (2022)        & 0.266 & 0.358 & 0.183 & 0.217 & 0.296 & 0.492 & 0.302 & 0.134  & 0.168 & 0.020 & 0.107 & 0.237 \\
                       & CutAddPaste (2024) & 0.363 & 0.458 & 0.192 & 0.700 & 0.280 & 0.896 & 0.482 & 0.321  & \textbf{0.620} & \underline{0.042} & 0.328 & 0.430 \\
                       & COUTA (2024)       & 0.414 & 0.421 & 0.362 & 0.312 & \textbf{0.389} & 0.626 & 0.421 & 0.658  & 0.221 & 0.010 & 0.296 & 0.379 \\
                       & CARLA (2025)       & \underline{0.523} &\underline{0.529} & \underline{0.511} & 0.721 & 0.293 & \underline{0.908} & \underline{0.581} & \underline{0.749}  & 0.308 & 0.039 & \underline{0.365} & \underline{0.509} \\ \cmidrule(lr){2-14}
Ours                  & GenIAS+CARLA       & \textbf{0.579} & \textbf{0.620} & \textbf{0.528} & \underline{0.755} & \underline{0.329} & \textbf{0.937} & \textbf{0.625} & \textbf{0.811}  & \underline{0.393} & \textbf{0.064} &\textbf{0.423} & \textbf{0.557} \\ \midrule
\end{tabular}
\end{adjustbox}
\vspace{-5pt}
\end{table*}

To illustrate GenIAS's enhanced anomaly generation compared to other anomaly injection-based baselines (US-I), we present t-SNE visualizations in Figure~\ref{fig:tsne}, comparing generated and real anomalies on Yahoo. Additional visualizations are provided in Figure \ref{fig:ext_tsne} in Appendix~\ref{sec:apd:vis}.

Impressively, GenIAS-generated anomalies align substantially more closely with the actual anomaly distribution compared to baseline methods, effectively covering true anomaly regions.
This is essential to provide anomaly-discriminative information for establishing more distinctive normal and anomaly representations in unsupervised TSAD training.
In contrast, baseline methods struggle to generate anomalies that adequately capture the actual anomaly distribution, as shown by large blue regions (actual anomalies) left uncovered by orange regions (generated anomalies).
For example, CARLA, the second-best method in overall ARP and TSAD performance, struggles on Yahoo, where its generated anomalies largely miss true anomaly regions. This highlights a key limitation of relying on handcrafted, predefined anomaly patterns: they risk distribution misalignment when assumptions deviate from reality, often resulting in disconnected or scattered anomaly regions. In contrast, GenIAS leverages controlled, regularized perturbation on the learned distribution to generate anomalies without pattern speculation, achieving better alignment and smoothness.

\subsection{Main TSAD Performance}\label{sec:tsad}
We report the TSAD performance of GenIAS and the competing methods in terms of the best F1 score in Table \ref{table:main}. Due to space constraints, results for the other four metrics discussed in Section~\ref{sec:metrics} are provided in Appendix~\ref{app:othermetrics}. GenIAS achieves significantly better overall TSAD performance compared to the baselines across metrics, demonstrating its superior effectiveness on both MTS and UTS. Specifically, among methods applicable to both types of time series, GenIAS achieves F1 improvements of 9.4\% and 29.5\% over the second- and third-best, respectively. A more detailed analysis by method category and dataset type reveals the following key observations. 

\noindent\textbf{Consistent effectiveness for both MTS and UTS}. On each type of dataset, GenIAS yields similar levels of performance improvement as observed in overall performance. For example, compared to the second-best model for both types, GenIAS improves F1 by 7.6\% on MTS and improves F1 by 15.9\%. This highlights its ability to generate informative and diverse anomalies that can mimic anomalous patterns within each dimension and the dependencies between dimensions. We also provide the average rank of all models in the form of critical difference diagrams \cite{demvsar2006statistical} for F1 and AUPR in both MTS and UTS (see Appendix~\ref{app:critical}). GenIAS achieved the highest rank in all four comparisons and demonstrated statistically significant improvements over all baselines in MTS for the F1 score.

\noindent\textbf{Enhanced Anomaly Generation}. GenIAS consistently outperforms US-I baselines by generating more realistic and diverse anomalous patterns. In particular, its performance gain over CARLA while only differing in their anomaly generation highlights the effectiveness of GenIAS’s tailored variational learning and novel perturbation strategies, which enforce more distinctive normal representations without relying on handcrafted anomaly types.

\noindent\textbf {Importance of anomaly generation in TSAD}. US-I injection methods generally outperform other unsupervised approaches and, in some cases, like CutAddPaste and CARLA, can even outperform the supervised baseline LSTM-VAE. This suggests that well-designed anomaly generation can produce realistic pseudo-anomalies that capture true anomalies, including those missed by labeled data. Moreover, it enables greater anomaly diversity, improving representation learning. Therefore, anomaly injection should be considered as an essential module in TSAD models.

\subsection{Ablation Study} \label{sec:ablation}
To better understand the contributions of different components and configurations in our proposed model, we conducted an ablation study.

\begin{table}[t]
\caption{Performance comparison of patching methods}
\label{tab:patch}
\centering
\begin{adjustbox}{width=0.75\columnwidth}
\begin{tabular}{@{}lcccccc@{}}
\toprule
\textbf{Dataset} & \multicolumn{2}{c}{\textbf{wo Patch}} & \multicolumn{2}{c}{\textbf{Length-driven}} & \multicolumn{2}{c}{\textbf{Deviation-based}} \\
\cmidrule(lr){2-3} \cmidrule(lr){4-5} \cmidrule(lr){6-7}
& F1 & AUPR & F1 & AUPR & F1 & AUPR \\ \midrule
MSL   & 0.535 & 0.466 & 0.547 & 0.488 & \textbf{0.579} & \textbf{0.529} \\
SMAP  & 0.533 & 0.425 & 0.571 & 0.429 & \textbf{0.620} & \textbf{0.467} \\
Yahoo & 0.755 & 0.776 & 0.741 & 0.766 & \textbf{0.811} & \textbf{0.794} \\ 
\bottomrule
\end{tabular}
\end{adjustbox}
\vspace{-5pt}
\end{table}

\begin{figure}[t]
    \centering
    \includegraphics[width=0.95\linewidth]{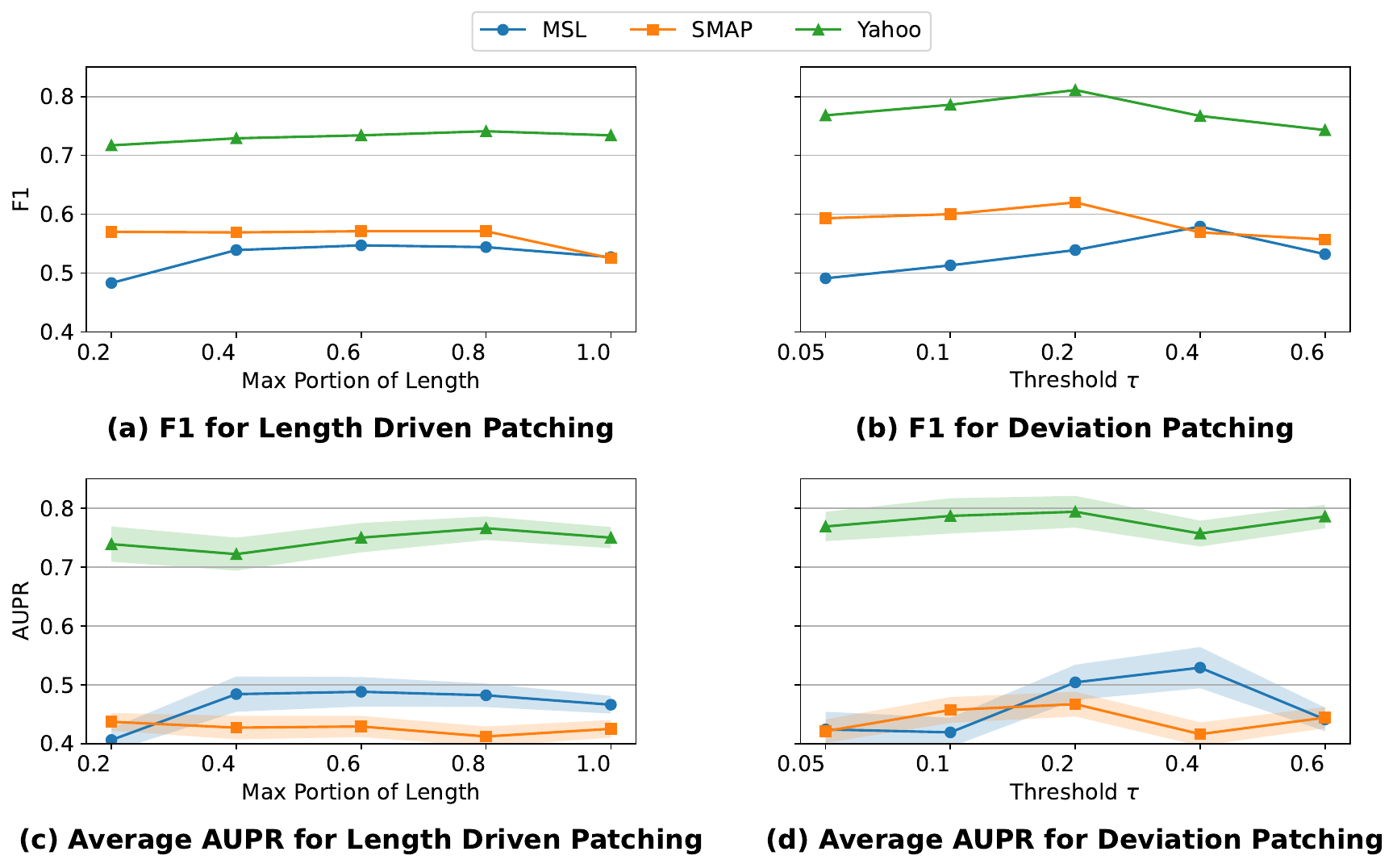}
    \caption{F1 and AUPR for different patching methods.}
    \label{fig:patch_res}
    \vspace{-5pt}
\end{figure}

\noindent\textbf{Impact of the Patch Selection.}
Figure~\ref{fig:patch_res} presents our analysis of F1 and AUPR for two patching methods—Length-driven and Deviation patching—evaluated across varying patch lengths and thresholds. The x-axis indicates the varying portion of patch length and the threshold \( \tau \) values for these approaches. The best results from this parameter study are compared in Table~\ref{tab:patch}, demonstrating that Deviation patching consistently outperforms other approaches, achieving the highest F1 and AUPR scores across all datasets. For example, on SMAP, it improves F1 by 0.044 and AUPR by 0.063 compared to the baseline without patching (wo Patch). \textbf{Length-driven} patching, which injects anomalies based on predefined temporal lengths, shows moderate improvements over the baseline in most scenarios but slightly underperforms on Yahoo dataset, where it reduces F1 and AUPR scores. In contrast, the strong performance of Deviation patching (Section~\ref{sec:patching}) highlights its effectiveness in aligning generated anomalies with true anomaly distributions, making it a more robust and effective approach for anomaly generation and detection.

\begin{figure}[t]
\centering
    \begin{subfigure}{0.23\textwidth}
    \label{fig:f1_prior}
        \includegraphics[width=\textwidth]{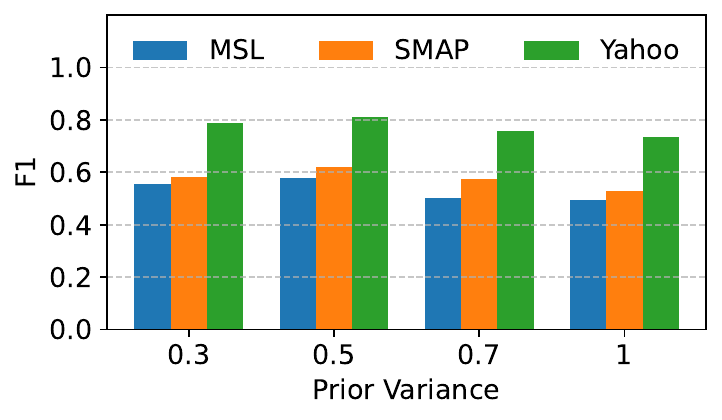}
    \end{subfigure}
    \begin{subfigure}{0.23\textwidth}
    \label{fig:aupr_prior}
        \includegraphics[width=\textwidth]{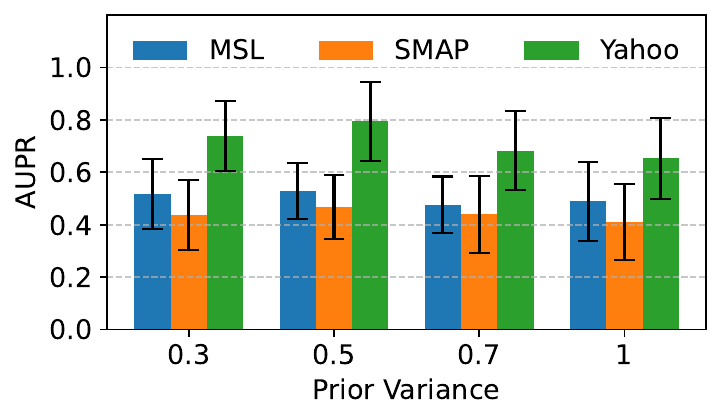}  
    \end{subfigure}
    \caption{F1 and AUPR for different prior variance ($\sigma_{\text{prior}}^2$)}
    \label{fig:prior}
    \vspace{-10pt}
\end{figure}

\noindent\textbf{Effectiveness of the Prior Variance.}
Figure~\ref{fig:prior} analyzes the impact of prior variance on anomaly detection performance, showing consistent trends across datasets and low sensitivity to this parameter. A prior variance of 0.5 yields the best F1 scores across all datasets (MSL: 0.579, SMAP: 0.62, Yahoo: 0.811), suggesting an effective balance between bias and variance. However, as the prior variance increases beyond 0.5, performance declines, indicating that overly broad priors 
reduce model specificity. Similarly, the AUPR results follow this trend, with 0.5 showing superior means and the smallest standard deviation (std), emphasizing the robustness of this setting. The growing std at higher variances further supports that large priors can lead to unstable and less reliable performance. Overall, a prior variance of 0.5 emerges as the most effective choice, yielding a balance between accuracy and stability for both F1 and AUPR.

\begin{figure}[t]
    \centering
    \includegraphics[width=0.95\linewidth]{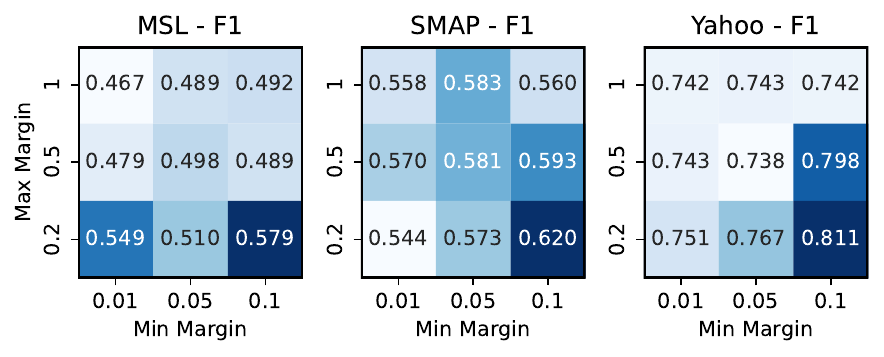}
    \caption{Effectiveness of min-max margins.}
    \label{fig:min-max}
    \vspace{-10pt}
\end{figure}

\noindent\textbf{Effectiveness of the Min-Max Margin.}
The heatmaps in Figure~\ref{fig:min-max} illustrate the impact of min-max margins (\(\delta_{\text{min}}, \delta_{\text{max}}\)) on F1 across MSL, SMAP, and Yahoo datasets. Smaller max margins (\(\delta_{\text{max}} = 0.2\)) and moderate min margins (\(\delta_{\text{min}} = 0.1\)) consistently yield the best results, with peak F1 scores of 0.579, 0.62, and 0.811 for MSL, SMAP, and Yahoo, respectively. 
These results validate the design rationale in Section~\ref{sec:pert_loss}, where \(\delta_{\text{min}}\) enforces distinctness and \(\delta_{\text{max}}\) maintains realism. Larger margins (e.g., \(\delta_{\text{max}} = 1\)) reduce performance, as seen in lower F1 scores of MSL and Yahoo, emphasizing the importance of balancing margins.

\section{Conclusion and Future Work}
We introduced GenIAS, a generative anomaly injection framework designed to enhance time series anomaly detection by creating realistic and diverse anomalies. By perturbing latent representations of a generative model trained on normal data, GenIAS produces anomalies that are similar to real-world anomalies. Integrated into a TSAD model, it boosts anomaly detection performance by enriching training data with high-quality synthetic anomalies.

For future work, we will explore adaptive perturbation and patching strategies, along with modelling inter-dimensional correlations, to further improve synthetic anomaly generation for TSAD.

\section*{Impact Statement}
GenIAS improves TSAD by enabling realistic and controllable synthetic anomaly generation using variational neural models. This alleviates the difficulty of training robust algorithms in application domains with few labeled anomalies, particularly for complex multivariate time series with subtle anomalous patterns.


\bibliography{ref}
\bibliographystyle{icml2026}

\newpage
\appendix
\onecolumn
\section{Datasets Descriptions}
\label{app:dataset}
Specifically, five popular TSAD datasets are included to evaluate multivariate TSAD. Among these, \textbf{NASA Datasets} — The \textbf{Mars Science Laboratory (MSL)} and \textbf{Soil Moisture Active Passive (SMAP)} datasets\footnote{\url{https://www.kaggle.com/datasets/patrickfleith/nasa-anomaly-detection-dataset-smap-msl}}~\citep{hundman2018detecting} are collected from NASA spacecraft. These datasets include anomaly information derived from incident reports for spacecraft monitoring systems. \textbf{Server Machine Dataset (SMD)}\footnote{\url{https://github.com/NetManAIOps/OmniAnomaly/tree/master/ServerMachineDataset}}~\citep{su2019robust} comprises data from 28 servers over 10 days. Normal data is observed during the first 5 days, while anomalies are injected into the remaining 5 days. \textbf{Secure Water Treatment (SWaT)}\footnote{\url{https://itrust.sutd.edu.sg/testbeds/secure-water-treatment-swat/}}~\citep{mathur2016swat} contains data from a water treatment platform monitored by 51 sensors over 11 days. A total of 41 anomalies were deliberately generated during the last 4 days using various attack scenarios. \textbf{SWAN}\footnote{\url{https://dataverse.harvard.edu/api/access/datafile/:persistentId?persistentId=doi:10.7910/DVN/EBCFKM/K9AOSI}}~\citep{
DVN/EBCFKM_2020} is a publicly accessible, MTS benchmark derived from solar photospheric vector magnetograms in the Spaceweather HMI Active Region Patch series. \textbf{GECCO}\footnote{\url{http://www.spotseven.de/gecco/gecco-challenge/gecco-challenge-2018/}}~\citep{moritz2018gecco} is a dataset on drinking water quality adapted for applications of the `Internet of Things'. It was introduced during the 2018 Genetic and Evolutionary Computation Conference.

In addition to the above MTS datasets, we also evaluated UTS datasets: \textbf{UCR}\footnote{\url{https://www.cs.ucr.edu/~eamonn/time_series_data_2018/}}~\citep{wu2021current} was provided for the Multi-dataset Time Series Anomaly Detection Competition at KDD2021. This dataset includes 250 UTS subdatasets, each containing one anomaly within its subsequences.
\textbf{Yahoo}\footnote{\url{https://webscope.sandbox.yahoo.com/catalog.php?datatype=s&did=70}}~\citep{yahoods} consists of 367 hourly time series with labels. Our focus is on the A1 benchmark, which includes 67 UTS representing “real” production traffic data from Yahoo properties.
\textbf{Key Performance Indicators (KPI)}\footnote{\url{https://github.com/NetManAIOps/KPI-Anomaly-Detection}} are sourced from real-world Internet company scenarios. It contains service and machine KPIs, such as response time, page views, CPU usage, and memory usage.

For all datasets, we adopt their original train/test splits, with the training data being unlabeled, except for the Yahoo and GECCO datasets, where we adopt a 50/50 split for training and testing.

\section{List of Symbols and Notations}

You can find the list of symbols and notations, and their definitions in Table \ref{tab:symbols_notations}.

\begin{table}[tp]
    \centering
    \small
    \caption{List of Symbols and Notations}
    \begin{tabular}{ll}
        \toprule
        \textbf{Symbol} & \textbf{Definition} \\
        \midrule
        \( \mathcal{D} \) & Time series dataset \\
        \( T \) & Number of time steps in a time series window \\
        \( D \) & Number of time series dimensions \\
        \( \mathbf{X} \) & Time series window of shape \( T \times D \) \\
        \( \widetilde{\mathbf{X}} \) & Generated anomalous time series sample \\
        \( \hat{\mathbf{X}} \) & Reconstructed normal sample from \( \mathbf{z} \) \\  
        \( \Gamma \) & Generative mapping from normal to anomalous samples \\
        \( \phi \) & Parameters of the encoder \\
        \( \theta \) & Parameters of the decoder \\
        \( \mathbf{z} \) & Latent representation of a normal sample \\
        \( \widetilde{\mathbf{z}} \) & Perturbed latent representation of an anomaly \\
        \( L \) & Latent space dimensionality\\
        \( p(\mathbf{X}) \) & Normal data distribution \\
        \( p(\widetilde{\mathbf{X}}) \) & Anomalous data distribution \\  
        \( q(\mathbf{z} | \mathbf{X}) \) & Latent distribution of normal data \\
        \( p(\widetilde{\mathbf{z}}) \) & Shifted latent distribution for anomalies \\
        \( \mu \) & Mean of the latent distribution \\
        \( \sigma \) & Standard deviation of the latent distribution \\
        \( \sigma_{\text{prior}} \) & Prior standard deviation for KL divergence loss \\
        \( \psi \) & Learned perturbation scale \\
        \( \delta_{\text{min}} \) & Minimum margin for anomaly distinctness \\
        \( \delta_{\text{max}} \) & Maximum allowable deviation for realism \\ 
        \( \tau \) & Threshold for deviation-based patching \\
        \bottomrule
    \end{tabular}
    \label{tab:symbols_notations}
\end{table}

\section{Algorithms Overview} \label{app:code}
The pseudo-code for GenIAS is provided in Algorithm~\ref{alg:genias}, while the deviation-based patching approach is detailed in Algorithm~\ref{alg:patching}.

\begin{algorithm}[h]
\caption{End-to-End Training of GenIAS}
\label{alg:genias}
\KwIn{TS dataset \( \mathcal{D}\), prior standard deviation \(\sigma_{\text{prior}}\), loss coefficients \(\alpha\), \(\beta\), \(\gamma\), \(\zeta\).}
\KwOut{Trained encoder parameters \(\phi\), decoder parameters \(\theta\), perturbation scale \(\psi\).}

\nl \textbf{Initialize} encoder \(q_\phi\), decoder \(g_\theta\), latent space parameters \(\mu\), \(\sigma\), and perturbation scale \(\psi\)\;
\nl \While{training not converged}{
\nl     Sample a batch \(\{\mathbf{\text{X}}_i\}_{i=1}^{B} \subset \mathcal{D}\) of size \(B\)\;
\nl     Encode: Obtain latent parameters \(\mu_i\), \(\sigma_i\) by \(q_{\phi}(\mathbf{\text{X}}_i)\)\;
\nl     Sample latent vectors \(\mathbf{\text{z}}_i \sim \mathcal{N}({\mu}_i, \operatorname{diag}({\sigma}_i^2))\)\;
\nl     Generate perturbed latent vectors \(\widetilde{\mathbf{z}}_i\) \hfill $\triangleright$ Eq. (\ref{eq:z_perturbed})\\
\nl     Decode reconstructed samples: \(\hat{\mathbf{X}}_i = g_\theta(\mathbf{z}_i)\)\;
\nl     Decode anomalous samples: \(\hat{\mathbf{X}}_i\ = g_\theta(\widetilde{\mathbf{z}}_i)\)\;
\nl     Compute the reconstruction loss $\mathcal{L}_{\text{recon}}$; \hfill $\triangleright$ Eq. (\ref{eq:reconstruction_loss})\\
\nl     Compute the perturbation loss $\mathcal{L}_{\text{perturb}}$; \hfill $\triangleright$ Eq. (\ref{eq:perturbation_loss})\\
\nl     Compute the zero-perturb loss $\mathcal{L}_{\text{zero-perturb}}$; \hfill $\triangleright$ Eq. (\ref{eq:zero_perturbation_loss})\\
\nl     Compute the compact KL loss $\mathcal{L}_{\text{comp-KL}}$; \hfill $\triangleright$ Eq. (\ref{eq:kl_divergence_loss})\\
\nl     $\mathcal{L}_{\text{total}} = \alpha \cdot \mathcal{L}_{\text{recon}} + \beta \cdot \mathcal{L}_{\text{perturb}} + \gamma \cdot \mathcal{L}_{\text{zero-perturb}} + \zeta \cdot \mathcal{L}_{\text{comp-KL}}$;\\
\nl     Update the model parameters \(\phi\), \(\theta\) and \(\psi\)\  by minimizing \(\mathcal{L}_{\text{total}}\)\;
}
\nl \Return \(\phi\), \(\theta\), \(\psi\)\;
\end{algorithm}

\begin{algorithm}[h]
\caption{Patching with GenIAS}
\label{alg:patching}
\KwIn{Pretrained GenIAS model's encoder \(q_\phi\), decoder \(g_\theta\) and perturbation scales \(\psi\), time series window \( \mathbf{X} \), threshold scaling factor \( \tau \).}
\KwOut{Patched anomalous time series windows \( \tilde{\mathbf{X}}_{\text{patched}} \).}

\textbf{Phase 1: Anomaly Generation}\\
\nl Load encoder \(q_\phi\), decoder \(g_\theta\), and perturbation scales \(\psi\).

\nl     Encode: Obtain latent parameters \(\mu\), \(\sigma\) by \(q_{\phi}(\mathbf{\text{X}})\)\;

\nl     Generate perturbed latent vectors \(\widetilde{\mathbf{z}}\) \hfill $\triangleright$ Eq. (\ref{eq:z_perturbed})\\

\nl     Decode \( \widetilde{\mathbf{z}} \) to generate preliminary anomalies: \( \tilde{\mathbf{X}} = g_\theta(\tilde{\mathbf{z}}) \)\;

\textbf{Phase 2: Deviation-Based Patching}\\
\nl \ForEach{dimension \( d \) \text{in} \( \mathbf{X} \)}{
\nl     \( \text{deviation}_d = \left| \mathbf{X}_d - \tilde{\mathbf{X}}_d \right|^2 \)\;
\nl     \( \text{amplitude}_d = \max(\mathbf{X}_d) - \min(\mathbf{X}_d) \)\;
\nl     Generate patched output for dimension \( d \):
    \[
    \tilde{\mathbf{X}}_{\text{patched}, d} = 
    \begin{cases} 
    \tilde{\mathbf{X}}_d & \text{if deviation}_d > \tau \cdot \text{amplitude}_d, \\
    \mathbf{X}_d & \text{otherwise}.
    \end{cases}
    \]
}

\nl \Return \( \tilde{\mathbf{X}}_{\text{patched}} \) as the patched anomalous TS windows.
\end{algorithm}

\section{Proofs of Theorem} \label{app:proofs}
\subsection{Notation and Setup}
Assume that the latent encoder produces a Gaussian posterior
\begin{equation}
q_{\phi}(z \mid {\mathbf{X}}) = \mathcal{N} \left( \mu, \operatorname{diag}(\sigma^2) \right),
\end{equation}
and that the prior is chosen as
\begin{equation}
p(z) = \mathcal{N} \left( 0, \sigma^2_{\text{prior}} I \right),
\end{equation}
where $\sigma^2_{\text{prior}}$ is a hyperparameter and $\mu$ is a vector $\mu=(\mu_1,...,\mu_L)$. In the following, we consider the effect of using a prior $\sigma^2_{\text{prior}} < 1$ in our Compact KL Loss (Equation~\ref{eq:kl_divergence_loss}).

\subsection{Prerequisites}
\begin{lemma}[Compact Latent Space in Compact KL Loss] \label{lemma:compact_full}
Consider the KL divergence between the estimated posterior and the prior distributions for a latent dimension $j$. If the compact KL loss -Equation~\ref{eq:kl_divergence_loss}- is minimized (or its gradient dominates locally) with a prior variance $\sigma^2_{\text{prior}} < 1$, then for each latent variable $z_j$ the optimal posterior variance satisfies $\sigma^2_j = \sigma^2_{\text{prior}}$. Hence, the variance of latent representations for normal samples is forced to be less than 1, i.e., more ``compact''.
\end{lemma}

\begin{proof}
For a one-dimensional Gaussian posterior $\mathcal{N}(\mu, \sigma^2)$ and a prior $\mathcal{N}(0, \sigma^2_{\text{prior}})$, the KL divergence is
\begin{equation}
D_{\text{KL}}(q(z \mid \mathbf{X}) \parallel p(z)) = \frac{1}{2} \left( \frac{\sigma^2}{\sigma^2_{\text{prior}}} + \frac{\mu^2}{\sigma^2_{\text{prior}}} - 1 - \log \frac{\sigma^2}{\sigma^2_{\text{prior}}} \right).
\end{equation}

If we minimize the KL term with respect to $\sigma^2$ (ignoring the contribution from the reconstruction term and assuming $\mu$ is fixed), we define a function $f$ that takes the part that depends on $\sigma^2$:
\begin{equation}
f(\sigma^2) = \frac{1}{2} \left( \frac{\sigma^2}{\sigma^2_{\text{prior}}} - \log \frac{\sigma^2}{\sigma^2_{\text{prior}}} \right).
\end{equation}

Taking the derivative,
\begin{equation}
\frac{d f}{d \sigma^2} = \frac{1}{2} \left( \frac{1}{\sigma^2_{\text{prior}}} - \frac{1}{\sigma^2} \right).
\end{equation}

Setting the derivative to zero yields
\begin{equation}
\frac{1}{\sigma^2_{\text{prior}}} - \frac{1}{\sigma^2} = 0 \quad \Longrightarrow \quad \sigma^2 = \sigma^2_{\text{prior}}.
\end{equation}

Thus, when $\sigma^2_{\text{prior}} < 1$, the optimal posterior variance (in the sense of minimizing the KL divergence) is less than 1, leading to a more compact latent representation.
\end{proof}

\begin{lemma}[Reduced Reconstruction Error for Normal Samples] \label{lemma-recon-norm}
Assume that the decoder $g_\theta(z)$ is trained on normal data. If the latent representations of normal samples are forced to be compact (i.e. the variance $\sigma^2(\mathbf{X})$ is reduced), then a first-order Taylor expansion argument shows that the additional error due to sampling variability is reduced, thereby lowering the expected reconstruction error.
\end{lemma}

\begin{proof}
For an input $\mathbf{X}$ with latent variable $z \sim q_\phi(z \mid \mathbf{X})$ and decoder $g_\theta(z)$, the reconstruction error is given by
\begin{equation}
\text{MSE}(\mathbf{X}) = \mathbb{E}_{z \sim q_\phi(z \mid \mathbf{X})} \left[ \| \mathbf{X} - g_\theta(z) \|^2 \right].
\label{eq:reconappendix}
\end{equation}

A first-order Taylor expansion of $g_\theta(z)$ around the mean $\mu(\mathbf{X})$ gives
\begin{equation}
g_\theta(z) \approx g_\theta(\mu(\mathbf{X})) + J_{g_\theta} (\mu(\mathbf{X})) (z - \mu(\mathbf{X})),
\label{eq:jacob}
\end{equation}
where $J_{g_\theta}$ is the Jacobian of $g_\theta$ at $\mu(\mathbf{X})$. We substitute Equation~\ref{eq:jacob} in Equation~\ref{eq:reconappendix}. While $\mu(\mathbf{X})$ stays the same in both Compact KL and KL cases, $z-\mu(\mathbf{X})$ varies in these two cases. Hence, we focus on the variable term in this substitution, i.e., $\mathbb{E} \left[ \| J_{g_\theta} (z - \mu(\mathbf{X})) \|^2 \right]$. 
In a multi-dimensional latent space, if the posterior covariance is $\Sigma(\mathbf{X}) = \operatorname{diag}(\sigma_1^2(\mathbf{X}), \dots, \sigma_d^2(\mathbf{X}))$, then the second-order contribution to the reconstruction error is approximately
\begin{equation}
\mathbb{E} \left[ \| J_{g_\theta} (z - \mu(\mathbf{X})) \|^2 \right] = \operatorname{trace} \left( J_{g_\theta}^T J_{g_\theta} \Sigma(\mathbf{X}) \right).
\label{eq:msejacob}
\end{equation}

When the compact KL loss forces each $\sigma_j^2(\mathbf{X})$ to be small (namely $\sigma^2_{\text{prior}} < 1$), the contribution of the variability term
\begin{equation}
\operatorname{trace} \left( J_{g_\theta}^T J_{g_\theta} \Sigma(\mathbf{X}) \right)
\end{equation}
is reduced. Hence, the overall reconstruction error decreases. Figure \ref{fig:mse_hist} shows the impact of $\sigma^2_{\text{prior}}$ on the reconstruction MSE distribution for entity C-1 in MSL dataset.
\end{proof}


\subsection{Proof of Proposition 1} \label{app:proof_t_1}
\textbf{Statement:} \textit{Let a VAE be trained on normal data with a latent prior \( \mathcal{N}(0, \sigma^2_{\text{prior}}) \), where \( \sigma_{\text{prior}} < 1 \). The encoder posterior for normal samples follows \( \mathcal{N}(\mu, \sigma^2_{\text{normal}}) \), with compactness enforced by KL regularization. If the encoder \( \phi \) applies a perturbation scale \( \psi \), inflating variance as  
\(
\sigma^2_{\text{anom}} = \psi \sigma^2_{\text{normal}}, \quad \text{with } \psi > 1,
\)
then the KL divergence between normal and anomalous latent distributions strictly increases compared to the case where \( \sigma_{\text{prior}} \geq 1 \).}  


\begin{proof}
For generated anomalous samples, assume that the latent posterior is
\begin{equation}
q_{\text{normal}}(z) = \mathcal{N}(\mu_{\text{normal}}, \operatorname{diag}(\sigma^2_{\text{normal}})),
\end{equation}
with $\sigma^2_{\text{normal}} \approx \sigma^2_{\text{prior}}$ (by Lemma~\ref{lemma:compact}). For anomalous inputs, suppose the encoder yields an inflated variance so that
\begin{equation}
q_{\text{anom}}(z) = \mathcal{N}(\mu_{\text{anom}}, \operatorname{diag}(\psi \sigma^2_{\text{normal}})),
\end{equation}
with $\psi > 1$. We assume that the means are similar, i.e. $\mu_{\text{anom}} \approx \mu_{\text{normal}}$, so that the primary difference is in the variances.

We aim to show that the KL divergence between the latent distributions of normal and anomalous samples \textbf{strictly increases} when the VAE prior variance is \textbf{compact} (\(\sigma_{\text{prior}} < 1\)) compared to when no prior compactness is enforced (\(\sigma_{\text{prior}} \geq 1\)).

From the given equation, the KL divergence between the anomalous and normal posterior distributions in one dimension is:

\begin{equation}
\begin{split}
D_{\text{KL}} \left( \mathcal{N}(\mu, \psi \sigma^2_{\text{normal}}) \Big\| \mathcal{N}(\mu, \sigma^2_{\text{prior}}) \right) = \\
\frac{1}{2} \left( \frac{\psi \sigma^2_{\text{normal}}}{\sigma^2_{\text{prior}}} - 1 - \log \frac{\psi \sigma^2_{\text{normal}}}{\sigma^2_{\text{prior}}} \right).
\end{split}
\end{equation}

We compare two scenarios:
\begin{itemize}
    \item \textbf{Scenario A (Compact Latent Space):} \( \sigma_{\text{prior}} < 1 \), let \( \sigma_{\text{prior}} = \alpha \), where \( 0 < \alpha < 1 \).
    \item \textbf{Scenario B (Non-Compact Latent Space):} \( \sigma_{\text{prior}} \geq 1 \), specifically, we consider \( \sigma_{\text{prior}} = 1 \).
\end{itemize}

For \textbf{Scenario A} (\(\sigma_{\text{prior}} = \alpha\)):

\begin{equation}
D_{\text{KL}, A} = \frac{1}{2} \left( \frac{\psi \sigma^2_{\text{normal}}}{\alpha^2} - 1 - \log \frac{\psi \sigma^2_{\text{normal}}}{\alpha^2} \right).
\end{equation}

For \textbf{Scenario B} (\(\sigma_{\text{prior}} = 1\)):

\begin{equation}
D_{\text{KL}, B} = \frac{1}{2} \left( \frac{\psi \sigma^2_{\text{normal}}}{1} - 1 - \log \frac{\psi \sigma^2_{\text{normal}}}{1} \right).
\end{equation}

Since \( 0 < \alpha < 1 \), we observe that the fraction \( \frac{\psi \sigma^2_{\text{normal}}}{\alpha^2} \) in \(D_{\text{KL}, A}\) is \textbf{strictly larger} than \( \frac{\psi \sigma^2_{\text{normal}}}{1} \) in \(D_{\text{KL}, B}\).

The function \( f(x) = x - 1 - \log x \) is \textbf{monotonically increasing} for \( x > 1 \), meaning:

\begin{equation}
f\left(\frac{\psi \sigma^2_{\text{normal}}}{\alpha^2}\right) > f\left(\frac{\psi \sigma^2_{\text{normal}}}{1}\right).
\end{equation}

Thus, we obtain:

\begin{equation}
D_{\text{KL}, A} > D_{\text{KL}, B}.
\end{equation}

Since \(D_{\text{KL}, A} > D_{\text{KL}, B}\), we conclude that the KL divergence between normal and anomalous latent distributions is strictly larger when the VAE is trained with a compact latent space (\(\sigma_{\text{prior}} < 1\)). 

This means that a compact latent space \textbf{amplifies the separation} between normal and anomalous samples under perturbation.
\end{proof}

\section{Details of Evaluation Metrics}\label{app:eval_metrics}
To support a consistent and meaningful evaluation of generated anomalies, we utilize a Deep SVDD \citep{pmlr-v80-ruff18a} representation with a fixed 128-dimensional embedding for all models. This representation provides a structured feature space. Additionally, we evaluate models using test datasets that contain labeled anomalies, allowing for a more reliable assessment of anomaly generation quality. Let \(\Omega\) denote the DeepSVDD model. The representation of real anomalous samples in the test dataset is given by \(V_{\text{real}} = \Omega(\text{X}_{\text{real anomalies}})\), while the representation of generated anomalous samples, produced by the anomaly generator model from the train dataset, is given by \(V_{\text{gen}} = \Omega(\tilde{{\text{X}}})\).

\subsection{Anomalous Representation Proximity (ARP)} \label{app:ARP}
ARP measures how well the generated anomalies cover the distribution of real anomalies in the representation space. For each real anomaly, we compute the distance to its closest generated anomaly. The inverse of the average of these minimum distances reflects how well the generated anomalies span the real anomaly space. Formally, if \(V_{\text{real}}\) and \(V_{\text{gen}}\) are the sets of real and generated anomaly representations, the ARP is defined as:
\begin{equation}
\text{ARP} = \frac{1}{1 + \frac{1}{|V_{\text{real}}|} \sum_{\mathbf{v}^r \in V_{\text{real}}} \min_{\mathbf{v}^g \in V_{\text{gen}}} d(\mathbf{v}^r, \mathbf{v}^g)},
\end{equation}
where \(d(\cdot, \cdot)\) is the Euclidean distance metric. A higher ARP indicates that the generated anomalies are more similar to real anomalies in the representation space, meaning better realism.

\subsection{Entropy-Based Diversity Index (EDI)} \label{app:EDI}
EDI quantifies how well a model generates a diverse set of anomalies by evaluating their distribution across the representation space. First, we collect the representations of all generated anomalies from all models into a global set:
\begin{equation}
V_{\text{all}} = \bigcup_{m=1}^{M} V_{\text{gen}}^{(m)},
\end{equation}
where \( V_{\text{gen}}^{(m)} = \{\mathbf{v}_1^{(m)}, \mathbf{v}_2^{(m)}, \dots, \mathbf{v}_{N_m}^{(m)}\} \) represents the set of generated anomalies for model \( m \), and \( M \) is the number of models being evaluated. Next, we partition the overall representation space into \( K \) non-overlapping regions \( \{R_1, R_2, \dots, R_K\} \) based on \( V_{\text{all}} \), ensuring that each anomaly belongs to exactly one region. 

For a specific model \( m \), the proportion of its generated anomalies in region \( R_i \) is defined as:
\begin{equation}
p_i^{(m)} = \frac{|\{\mathbf{v}_j^{(m)} \in V_{\text{gen}}^{(m)} \mid \mathbf{v}_j^{(m)} \in R_i\}|}{N_m},
\end{equation}
where \( N_m = |V_{\text{gen}}^{(m)}| \) is the total number of generated anomalies for model \( m \), and \( |\cdot| \) denotes the cardinality of a set.

The EDI for model \( m \) is then computed using \textit{Shannon Entropy}~\cite{shannon1948mathematical} as:
\begin{equation}
\text{EDI}^{(m)} = - \sum_{i=1}^{K} p_i^{(m)} \log p_i^{(m)}.
\end{equation}

A higher EDI indicates that a model generates anomalies that are well-distributed across the space, reflecting greater diversity rather than being concentrated in a few regions.

\subsection{TSAD Evaluation Metrics} \label{app:tsadmetrics}
We use the Best F1 score to measure the performance of TSAD models at the optimal threshold. AUPR and AUROC are holistic measures that assess model performance across various thresholds. All these metrics account for class imbalance, with AUPR placing greater emphasis on anomaly class performance, making it more reflective of anomaly detection capabilities.

In addition, the two recent metrics are chosen to evaluate TSAD performance at finer granularities: Affiliation F1 (derived from Affiliation-Precision and Affiliation-Recall) and PATE (Proximity-Aware Time series anomaly Evaluation that incorporates proximity-based weighting and buffer zones). Both metrics assess early or delayed detections. 

\section{Implementation Details} \label{sec:implementation}
All evaluations were performed on a system with an A40 GPU, 13 CPU cores, and 250 GB of RAM. We use the same hyperparameters for GenIAS across UTS and MTS datasets, as detailed in Table~\ref{tab:hyperparam}.
\begin{table}[h]
    \centering
    \small
    \caption{hyperparameter configurations used in GenIAS.}
    \vspace{-5pt}
    \begin{tabular}{ll}
        \toprule
        \textbf{Parameter} & \textbf{Value} \\
        \midrule
        \( T \) & 200 \\
        \( L \) & UTS = 50, MTS = 100 \\
        \( \sigma_{\text{prior}} \) & 0.5 \\
        \textit{lr} & initial=$10^{-4}$, with scheduled reduction\\
        \textit{epochs} & max=1000 \\
        \textit{kernel size} & 3\\
        \textit{dropout} & 0.1\\
        \textit{batch size} & 100 \\
        $\alpha$ & 1.0 \\
        $\beta$ & 0.1 \\
        $\gamma$ & UTS = 0.0, MTS = 0.01 \\
        $\zeta$ & 0.1 \\
        \bottomrule
    \end{tabular}
    \label{tab:hyperparam}
\end{table}

\section{Additional Results}
\subsection{Visualizing Compact KL Loss Effects}
Separation of reconstruction error in the input space is illustrated in Figure~\ref{fig:mse_hist}.
\begin{figure}[t]
    \centering
    \includegraphics[width=0.7\linewidth]{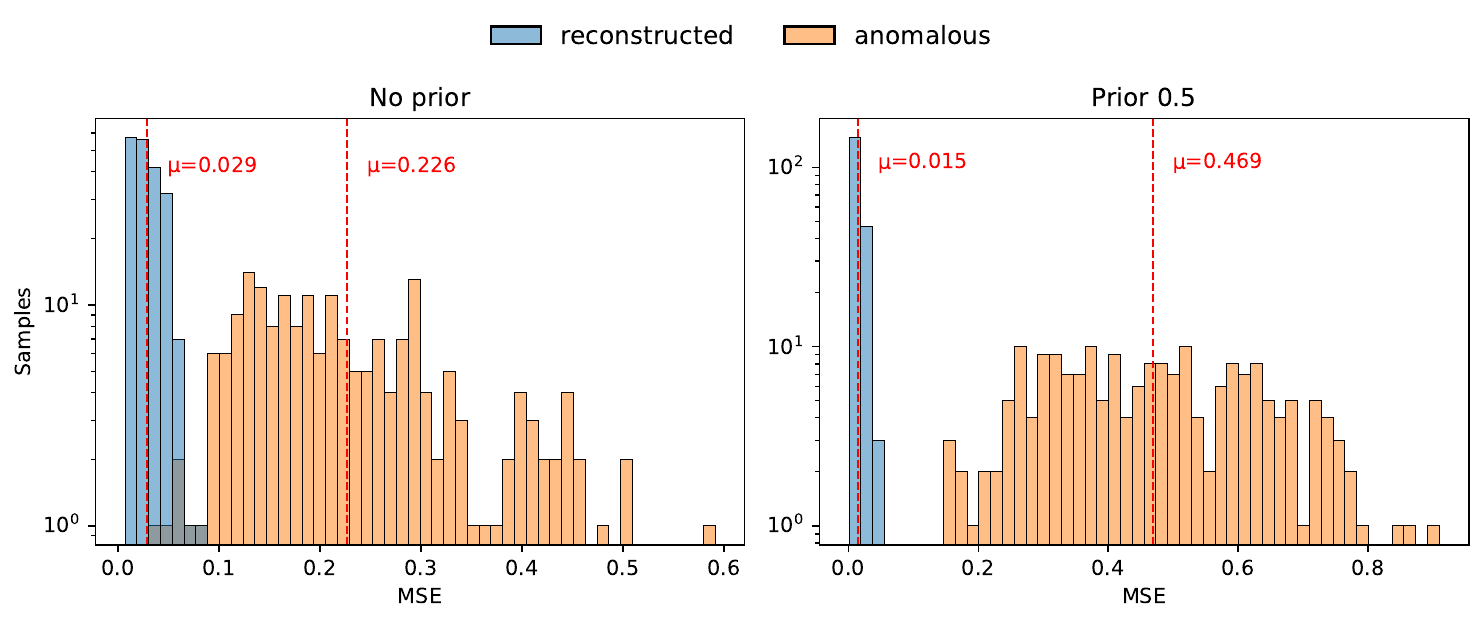}
    \caption{Separation of reconstruction error (anomaly score) in the input space with a prior of 0.5 (left) and without a prior (right) on entity C-1 from the MSL dataset. The y-axis is logarithmic.}
    \label{fig:mse_hist}
\end{figure}

\subsection{TSAD additional results}

\subsubsection{Critical Difference diagrams} \label{app:critical}
Statistical significance in critical difference diagrams was determined using a Friedman test followed by a post-hoc Wilcoxon signed-rank test with Holm correction \cite{demvsar2006statistical}. The critical difference diagrams in Figure~\ref{fig:cdc} show that GenIAS consistently ranks first across both multivariate (MTS) and univariate (UTS) time series settings for AUPR and F1 Score. In the MTS setting, GenIAS demonstrates a statistically significant improvement over other baselines, indicating its superior performance in capturing anomalies.
\begin{figure*}[t]
    \centering
    \begin{subfigure}{0.49\textwidth}
        \includegraphics[width=\textwidth]{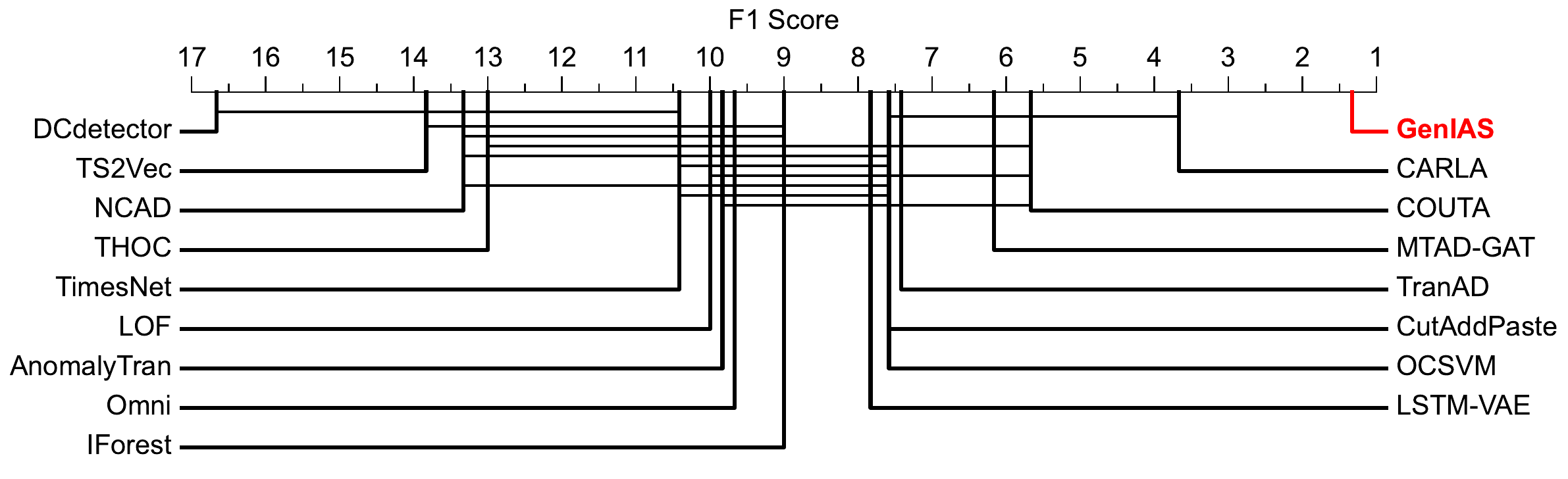}
        \label{fig:cdc-f1-mts}
        \vspace{-20pt}
        \caption{}
    \end{subfigure}
    \hfill
    \begin{subfigure}{0.49\textwidth}
        \includegraphics[width=\textwidth]{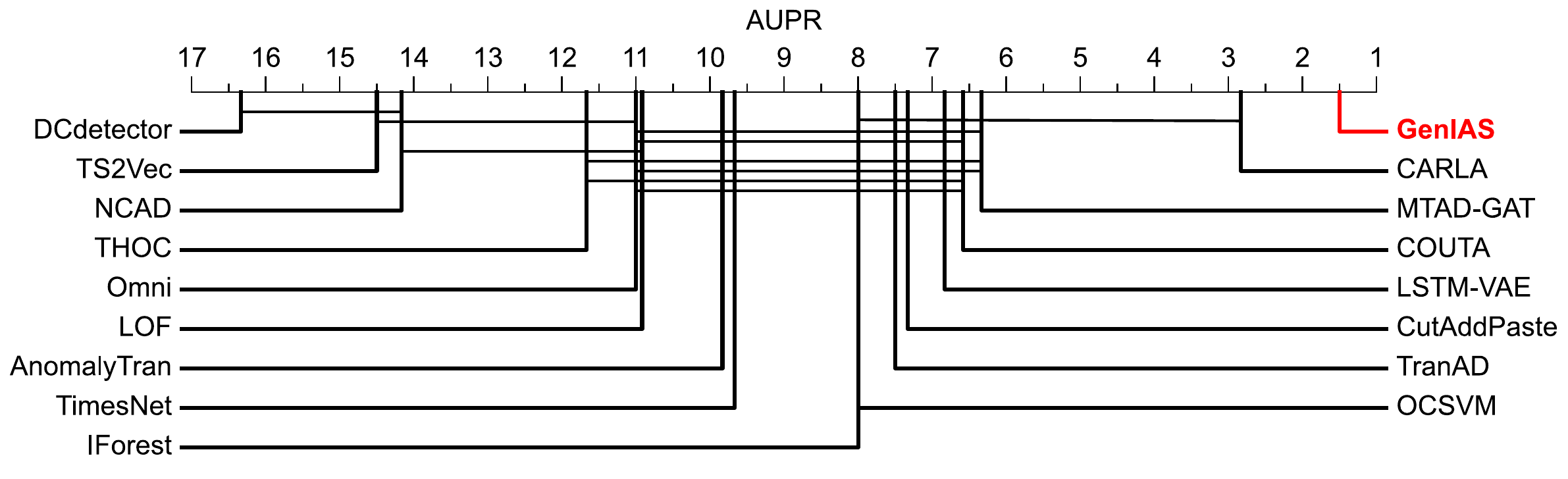}
        \vspace{-20pt}
        \caption{}
        \label{fig:cdc-pr-mts}
    \end{subfigure}
    \begin{subfigure}{0.48\textwidth}
        \includegraphics[width=\textwidth]{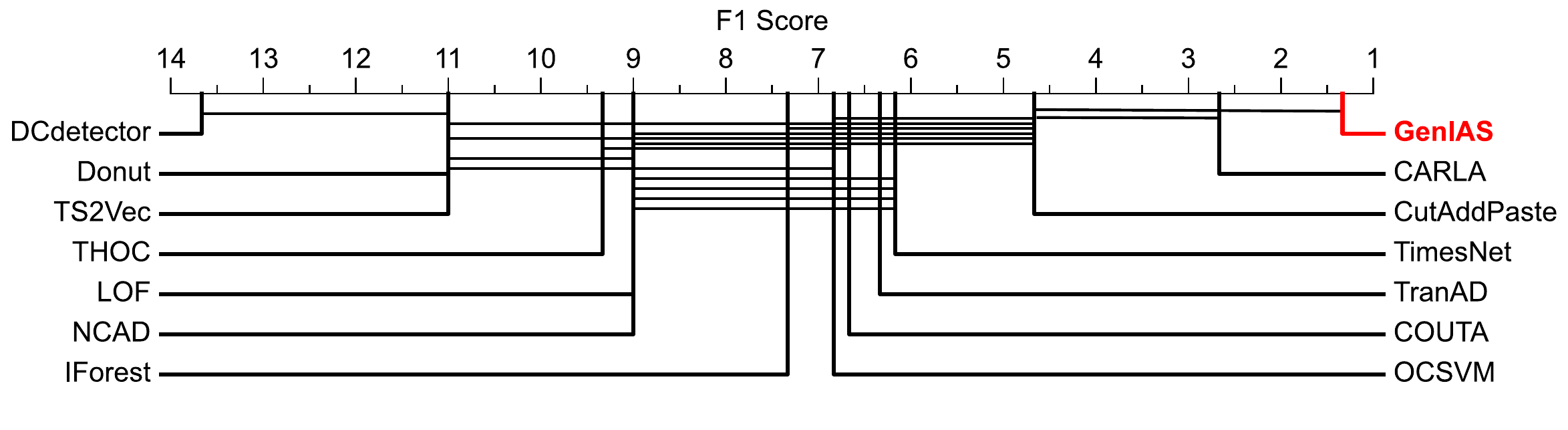}
        \label{fig:cdc-f1-uts}
        \vspace{-20pt}
        \caption{}
    \end{subfigure}
    \hfill
    \begin{subfigure}{0.48\textwidth}
        \includegraphics[width=\textwidth]{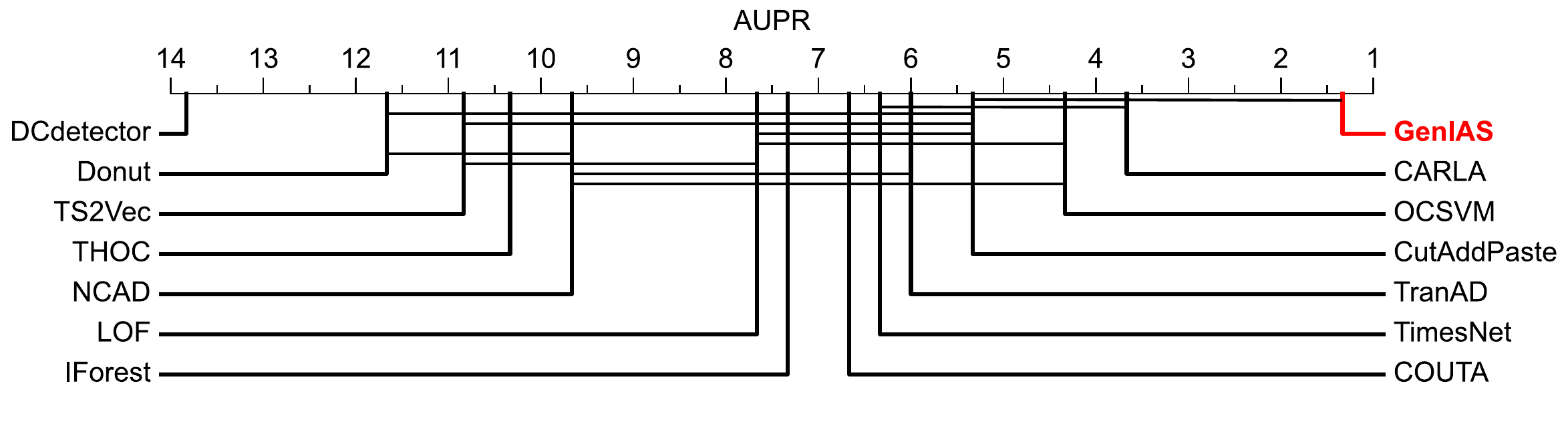}
        \label{fig:cdc-pr-uts}
        \vspace{-20pt}
        \caption{}
    \end{subfigure}
    \caption{Critical difference diagrams for F1 and AUPR: (a, b) represent multivariate models, while (c, d) showcase univariate models.}
    \label{fig:cdc}
\end{figure*}

\subsubsection{TSAD performance in the other metrics.} \label{app:othermetrics}
In Table~\ref{table:tsad_app}, we present the TSAD performance of GenIAS and the baseline methods in terms of three additional metrics: Affiliation F1 (Aff), PATE, and AUROC. In general, similar observations can be drawn as in Section~\ref{sec:tsad} - GenIAS achieves superior performance compared to the baseline methods. Specifically, for Aff and PATE, GenIAS achieves the best overall performance across both MTS and UTS benchmark datasets. In terms of AUROC, GenIAS achieves the second-best overall performance and best overall performance among neural network-based methods, while Isolation Forest, a well-known legacy shallow AD method, yields surprisingly good performance.


\begin{table*}[t]
\caption{TSAD performance of GenIAS and baseline methods in terms of Aff, PATE, and AUROC, with the best model in bold and the second-best model is underlined. U and M denote results unavailable for models limited to UTS or MTS, respectively.}
\centering
\label{table:tsad_app}
\begin{adjustbox}{width=0.87\textwidth}
\begin{tabular}{ccccccccccccccc}
\toprule
\multicolumn{1}{l}{}      & \multicolumn{1}{l}{}     & \multicolumn{1}{l}{} & \multicolumn{7}{c}{Multivariate}                                                                                     & \multicolumn{4}{c}{Univariate}                                    & \multicolumn{1}{l}{} \\ \cmidrule(lr){4-10} \cmidrule(lr){11-14} 
Metric                    & Cat.                     & Methods              & MSL            & SMAP           & SMD            & SWAT           & GECCO          & SWAN           & Avg-M          & Yahoo          & KPI            & UCR            & Avg-U          & Avg                  \\ \midrule
\multirow{18}{*}{AUPR} & \multirow{3}{*}{Classic} & OCSVM (1999)       & 0.202 & 0.256 & 0.327 & 0.696 & 0.163 & 0.608 & 0.375 & \underline{0.703}  & 0.300 & 0.028 & 0.344 & 0.365 \\ 
                       &                          & LOF (2000)         & 0.217 & 0.248 & 0.277 & 0.185 & 0.031 & 0.624 & 0.264 & 0.615  & 0.109 & 0.043 & 0.256 & 0.261 \\
                       &                          & IForest (2008)     & 0.182 & 0.190 & 0.278 &\textbf{0.739} & 0.162 & 0.739 & 0.382 & 0.614  & 0.288 & 0.021 & 0.308 & 0.357 \\ \cmidrule(lr){3-15}
                       & PU                      & LSTM-VAE (2018)    & 0.285 & 0.258 & 0.395 & 0.685 & 0.045 & 0.674 & 0.390 & M      & M     & M     & M     & M     \\ \cmidrule(lr){3-15}
                       & \multirow{9}{*}{US-WI}  & Donut (2018)       & U     & U     & U     & U     & U     & U     & U     & 0.264  & 0.078 & 0.015 & 0.119 & U    \\
                       &                          & Omni (2019)        & 0.149 & 0.115 & 0.365 & \underline{0.713} & 0.034 & 0.587 & 0.327 & M      & M     & M     & M     & M     \\
                       &                          & THOC (2020)        & 0.240 & 0.195 & 0.107 & 0.537 & 0.076 & 0.432 & 0.265 & 0.349  & 0.229 & 0.014 & 0.197 & 0.242 \\
                       &                          & MTAD-GAT(2020)     & 0.335 & 0.339 & 0.401 & 0.095 & 0.200 & 0.613 & 0.331 & M      & M     & M     & M     & M     \\
                       &                          & AnomalyTran (2021) & 0.236 & 0.264 & 0.273 & 0.680 & 0.059 & 0.568 & 0.347 & M      & M     & M     & M     & M     \\
                       &                          & TranAD (2022)      & 0.273 & 0.287 & 0.412 & 0.192 & 0.119 & 0.586 & 0.312 & 0.691  & 0.285 & 0.022 & 0.333 & 0.319 \\
                       &                          & TS2Vec (2022)      & 0.132 & 0.148 & 0.113 & 0.136 & 0.040 & 0.408 & 0.163 & 0.491  & 0.221 & 0.009 & 0.240 & 0.189 \\
                       &                          & TimesNet (2023)    & 0.283 & 0.208 & 0.385 & 0.082 & 0.189 & 0.508 & 0.276 & 0.671  & 0.237 & 0.032 & 0.313 & 0.288 \\
                       &                          & Dcdetector (2023)  & 0.129 & 0.124 & 0.043 & 0.126 & 0.011 & 0.326 & 0.127 & 0.041  & 0.018 & 0.009 & 0.023 & 0.092 \\ \cmidrule(lr){3-15}
                       & \multirow{4}{*}{US-I}                  & NCAD (2022)        & 0.146 & 0.156 & 0.096 & 0.106 & 0.135 & 0.407 & 0.174 & 0.067  & 0.089 & 0.103 & 0.086 & 0.145 \\ 
                       &                          & CutAddPaste (2024) & 0.242 & 0.322 & 0.153 & 0.582 & 0.159 & 0.808 & 0.378 & 0.116  & \textbf{0.602} & 0.188 & 0.302 & 0.352 \\
                       &                          & COUTA (2024)       & 0.247 & 0.248 & 0.400 & 0.159 & \textbf{0.264} & 0.685 & 0.334 & 0.555  & 0.282 & 0.049 & 0.295 & 0.321 \\
                       &                          & CARLA (2025)       & \underline{0.501} & \underline{0.448} & \underline{0.507} & 0.681 & 0.201 & \underline{0.814} & \underline{0.525} & 0.627  & 0.299 &\underline{0.247} & \underline{0.391} & \underline{0.481} \\ \cmidrule(lr){3-15}
                       & Ours                     & GenIAS+CARLA             & \textbf{0.529} & \textbf{0.467} & \textbf{0.512} & 0.704 & \underline{0.236} & \textbf{0.882} & \textbf{0.555} & \textbf{0.794}  & \underline{0.357} & \textbf{0.280} & \textbf{0.477} & \textbf{0.529} \\ 
\midrule
\multirow{18}{*}{Aff}     & \multirow{3}{*}{Classic} & OCSVM (1999)         & 0.742          & 0.761          & 0.733          & 0.197          & 0.425          & 0.711          & 0.595          & \textbf{0.923} & 0.761          & 0.710          & 0.798          & 0.663                \\
                          &                          & LOF (2000)           & 0.775          & 0.762          & 0.666          & 0.613          & 0.227          & 0.745          & 0.631          & 0.906          & 0.674          & 0.718          & 0.766          & 0.676                \\
                          &                          & IForest (2008)       & 0.719          & 0.728          & 0.699          & 0.535          & 0.259          & 0.824          & 0.627          & 0.907          & 0.760          & 0.686          & 0.784          & 0.680                \\ \cmidrule(lr){3-15}
                          & PU                       & LSTM-VAE (2018)      & 0.824          & 0.782          & 0.630          & 0.324          & 0.312          & 0.726          & 0.600          & M              & M              & M              & M              & M                    \\
                          & \multirow{9}{*}{US-WI}   & Donut (2018)         & U              & U              & U              & U              & U              & U              & U              & 0.777          & 0.682          & 0.685          & 0.715          & -                    \\
                          &                          & Omni (2019)          & 0.680          & 0.779          & 0.706          & 0.156          & 0.230          & 0.608          & 0.527          & M              & M              & M              & M              & M                    \\
                          &                          & THOC (2020)          & 0.778          & 0.781          & 0.677          & 0.679          & 0.321          & 0.713          & 0.658          & 0.796          & 0.676          & 0.642          & 0.705          & 0.674                \\
                          &                          & MTAD-GAT(2020)       & \underline{ 0.833}    & \underline{ 0.848}    & 0.731          & 0.355          & 0.308          & 0.745          & 0.637          & M              & M              & M              & M              & M                    \\
                          &                          & AnomalyTran (2021)   & 0.774          & 0.766          & 0.664          & 0.482          & 0.551          & 0.537          & 0.629          & M              & M              & M              & M              & M                    \\
                          &                          & TranAD (2022)        & 0.771          & 0.751          & 0.719          & 0.606          & 0.233          & 0.661          & 0.624          & 0.917          & 0.779          & 0.698          & 0.798          & 0.682                \\
                          &                          & TS2Vec (2022)        & 0.694          & 0.695          & 0.750          & \underline{ 0.700}    & 0.520          & 0.728          & 0.681          & 0.867          & 0.818          & 0.671          & 0.785          & 0.716                \\
                          &                          & TimesNet (2023)      & 0.769          & 0.710          & \textbf{0.795} & 0.693          & \underline{ 0.829}    & 0.710          & 0.751          & 0.917          & \underline{ 0.844}    & 0.712          & 0.824          & 0.775                \\
                          &                          & Dcdetector (2023)    & 0.706          & 0.680          & 0.528          & 0.693          & 0.671          & 0.710          & 0.665          & 0.737          & 0.666          & 0.671          & 0.691          & 0.674                \\ \cmidrule(lr){3-15}
                          & \multirow{4}{*}{US-I}    & NCAD (2022)          & 0.685          & 0.664          & 0.732          & 0.693          & \textbf{0.887} & 0.710          & 0.729          & 0.645          & 0.753          & 0.747          & 0.715          & 0.724                \\
                          &                          & CutAddPaste (2024)   & 0.762          & 0.753          & 0.693          & \textbf{0.764} & 0.659          & 0.875          & 0.751          & 0.832          & 0.818          & \textbf{0.820} & 0.823          & 0.775                \\
                          &                          & COUTA (2024)         & 0.785          & 0.762          & \underline{ 0.764}    & 0.645          & 0.592          & 0.731          & 0.713          & 0.906          & \textbf{0.851} & 0.730          & \underline{ 0.829}    & 0.752                \\
                          &                          & CARLA (2025)         & 0.816          & 0.808          & 0.733          & 0.592          & 0.756          & \underline{ 0.932}    & \underline{ 0.773}    & 0.874          & 0.701          & 0.805          & 0.793          & \underline{ 0.780}          \\ \cmidrule(lr){3-15}
                          & Ours                     & GenIAS+CARLA               & \textbf{0.851} & \textbf{0.856} & 0.731          & 0.596          & 0.763          & \textbf{0.935} & \textbf{0.789} & \underline{ 0.921}    & 0.788          & \underline{ 0.814}    & \textbf{0.841} & \textbf{0.806}       \\ \cmidrule(lr){1-15}
\multirow{18}{*}{PATE}    & \multirow{3}{*}{Classic} & OCSVM (1999)         & 0.267          & 0.270          & 0.363          & 0.698          & 0.191          & 0.704          & 0.416          & \underline{ 0.780}    & 0.343          & 0.033          & 0.385          & 0.405                \\
                          &                          & LOF (2000)           & 0.242          & 0.251          & 0.319          & 0.187          & 0.042          & 0.706          & 0.291          & 0.690          & 0.159          & 0.056          & 0.302          & 0.295                \\
                          &                          & IForest (2008)       & 0.195          & 0.197          & 0.306          & \textbf{0.743} & 0.191          & 0.799          & 0.405          & 0.735          & 0.325          & 0.027          & 0.362          & 0.391                \\ \cmidrule(lr){3-15}
                          & PU                       & LSTM-VAE (2018)      & 0.318          & 0.278          & 0.436          & 0.693          & 0.061          & 0.753          & 0.423          & M              & M              & M              & M              & M                    \\ \cmidrule(lr){3-15}
                          & \multirow{9}{*}{US-WI}   & Donut (2018)         & U              & U              & U              & U              & U              & U              & U              & 0.358          & 0.120          & 0.018          & 0.165          & -                    \\
                          &                          & Omni (2019)          & 0.143          & 0.162          & 0.422          & 0.702          & 0.076          & 0.709          & 0.369          & M              & M              & M              & M              & M                    \\
                          &                          & THOC (2020)          & 0.307          & 0.218          & 0.147          & 0.508          & 0.104          & 0.563          & 0.308          & 0.488          & 0.317          & 0.016          & 0.274          & 0.296                \\
                          &                          & MTAD-GAT(2020)       & 0.391          & 0.371          & 0.445          & 0.115          & 0.247          & 0.707          & 0.379          & M              & M              & M              & M              & M                    \\
                          &                          & AnomalyTran (2021)   & 0.267          & 0.283          & 0.332          & 0.683          & 0.112          & 0.685          & 0.394          & M              & M              & M              & M              & M                    \\
                          &                          & TranAD (2022)        & 0.238          & 0.279          & 0.458          & 0.192          & 0.141          & 0.688          & 0.333          & 0.777          & 0.408          & 0.025          & 0.403          & 0.356                \\
                          &                          & TS2Vec (2022)        & 0.143          & 0.150          & 0.137          & 0.147          & 0.055          & 0.530          & 0.194          & 0.636          & 0.247          & 0.144          & 0.342          & 0.243                \\
                          &                          & TimesNet (2023)      & 0.282          & 0.207          & 0.452          & 0.083          & 0.240          & 0.636          & 0.317          & 0.760          & 0.355          & 0.042          & 0.386          & 0.340                \\
                          &                          & Dcdetector (2023)    & 0.176          & 0.135          & 0.097          & 0.124          & 0.025          & 0.455          & 0.169          & 0.184          & 0.052          & 0.025          & 0.087          & 0.141                \\
                          & \multirow{4}{*}{US-I}    & NCAD (2022)          & 0.157          & 0.161          & 0.129          & 0.107          & 0.272          & 0.556          & 0.230          & 0.145          & 0.173          & 0.130          & 0.149          & 0.203                \\
                          &                          & CutAddPaste (2024)   & 0.269          & 0.347          & 0.200          & 0.613          & 0.252          & 0.842          & 0.421          & 0.237          & \textbf{0.741} & 0.262          & 0.413          & 0.418                \\
                          &                          & COUTA (2024)         & 0.270          & 0.244          & 0.447          & 0.159          & \textbf{0.285} & 0.756          & 0.360          & 0.684          & 0.335          & 0.054          & 0.358          & 0.359                \\
                          &                          & CARLA (2025)         & \underline{ 0.493}    & \underline{ 0.413}    & \underline{ 0.488}    & 0.655          & 0.193          & \underline{ 0.911}    & \underline{ 0.526}    & 0.695          & 0.378          & \underline{ 0.301}    & \underline{ 0.458}    & \underline{ 0.503}          \\ \cmidrule(lr){3-15}
                          & Ours                     & GenIAS+CARLA               & \textbf{0.571} & \textbf{0.516} & \textbf{0.553} & \underline{ 0.718}    & \underline{ 0.274}    & \textbf{0.941} & \textbf{0.596} & \textbf{0.838} & \underline{ 0.456}    & \textbf{0.331} & \textbf{0.542} & \textbf{0.578}       \\ \cmidrule(lr){1-15}
\multirow{18}{*}{AUC-ROC} & \multirow{3}{*}{Classic} & OCSVM (1999)         & 0.605          & 0.634          & 0.762          & 0.799          & 0.446          & 0.714          & 0.660          & \textbf{0.918} & 0.728          & 0.567          & 0.738          & 0.686                \\
                          &                          & LOF (2000)           & 0.610          & 0.584          & 0.708          & 0.660          & 0.370          & 0.789          & 0.620          & 0.890          & 0.604          & 0.575          & 0.690          & 0.643                \\
                          &                          & IForest (2008)       & 0.571          & 0.579          & 0.768          & \textbf{0.842} & \textbf{0.927} & \textbf{0.862} & \textbf{0.758} & \underline{ 0.915}    & 0.749          & 0.570          & 0.745          & \textbf{0.754}       \\ \cmidrule(lr){3-15}
                          & PU                       & LSTM-VAE (2018)      & 0.618          & 0.650          & 0.771          & 0.806          & 0.506          & 0.763          & 0.686          & M              & M              & M              & M              & M                    \\ \cmidrule(lr){3-15}
                          & \multirow{9}{*}{US-WI}   & Donut (2018)         & U              & U              & U              & U              & U              & U              & U              & 0.676          & 0.640          & 0.498          & 0.605          & -                    \\
                          &                          & Omni (2019)          & 0.387          & 0.570          & 0.794          & 0.811          & 0.430          & 0.696          & 0.615          & M              & M              & M              & M              & M                    \\
                          &                          & THOC (2020)          & 0.654          & 0.604          & 0.678          & 0.318          & 0.395          & 0.578          & 0.538          & 0.873          & \underline{ 0.799}    & 0.558          & 0.743          & 0.606                \\
                          &                          & MTAD-GAT(2020)       & 0.674          & 0.641          & \underline{ 0.812}    & \underline{ 0.822}    & 0.590          & 0.749          & \underline{ 0.715}    & M              & M              & M              & M              & M                    \\
                          &                          & AnomalyTran (2021)   & 0.641          & 0.653          & 0.764          & 0.809          & 0.698          & 0.678          & 0.707          & M              & M              & M              & M              & M                    \\
                          &                          & TranAD (2022)        & 0.632          & 0.645          & 0.808          & 0.677          & 0.519          & 0.677          & 0.660          & 0.889          & 0.763          & 0.515          & 0.722          & 0.681                \\
                          &                          & TS2Vec (2022)        & 0.503          & 0.517          & 0.496          & 0.548          & 0.501          & 0.573          & 0.523          & 0.698          & 0.567          & 0.501          & 0.589          & 0.545                \\
                          &                          & TimesNet (2023)      & 0.664          & 0.578          & \textbf{0.837} & 0.243          & \underline{ 0.924}    & 0.598          & 0.641          & 0.892          & 0.750          & 0.597          & 0.746          & 0.676                \\
                          &                          & Dcdetector (2023)    & 0.535          & 0.519          & 0.579          & 0.504          & 0.504          & 0.499          & 0.523          & 0.603          & 0.504          & 0.509          & 0.539          & 0.528                \\
                          & \multirow{4}{*}{US-I}    & NCAD (2022)          & 0.491          & 0.493          & 0.572          & 0.365          & 0.692          & 0.592          & 0.534          & 0.567          & 0.671          & 0.609          & 0.616          & 0.561                \\
                          &                          & CutAddPaste (2024)   & 0.673          & 0.631          & 0.631          & 0.777          & 0.789          & 0.464          & 0.661          & 0.684          & \textbf{0.810} & 0.617          & 0.704          & 0.675                \\
                          &                          & COUTA (2024)         & 0.661          & 0.606          & 0.811          & 0.643          & 0.660          & \underline{ 0.802}    & 0.697          & 0.911          & 0.787          & 0.559          & \textbf{0.752} & 0.716                \\
                          &                          & CARLA (2025)         & \underline{ 0.694}    & \underline{ 0.677}    & 0.657          & 0.783          & 0.696          & 0.510          & 0.670          & 0.543          & 0.604          & \underline{ 0.735}    & 0.627          & 0.655                \\  \cmidrule(lr){3-15}
                          & Ours                     & GenIAS+CARLA               & \textbf{0.745} & \textbf{0.718} & 0.633          & 0.786          & 0.811          & 0.536          & 0.705          & 0.778          & 0.691          & \textbf{0.778} & \underline{ 0.749}    & \underline{ 0.720}     \\ \bottomrule     
\end{tabular}
\end{adjustbox}
\end{table*}

\begin{figure*}[t]
\centering
    \begin{subfigure}{0.8\textwidth}
        \includegraphics[width=\textwidth]{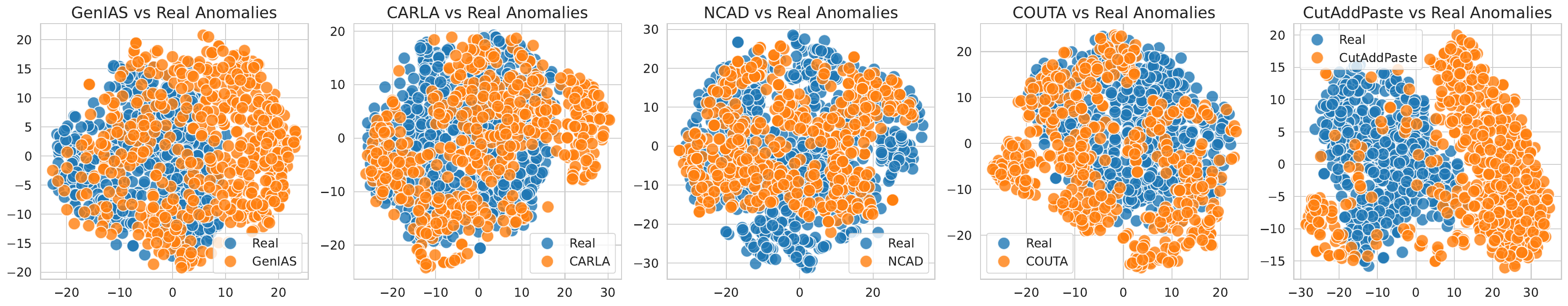}
        \label{fig:smap_tsne}
        \vspace{-15pt}
        \caption{SMAP - Entity E5}
    \end{subfigure}
    \par\vspace{5pt}
    \begin{subfigure}{0.8\textwidth}
        \includegraphics[width=\textwidth]{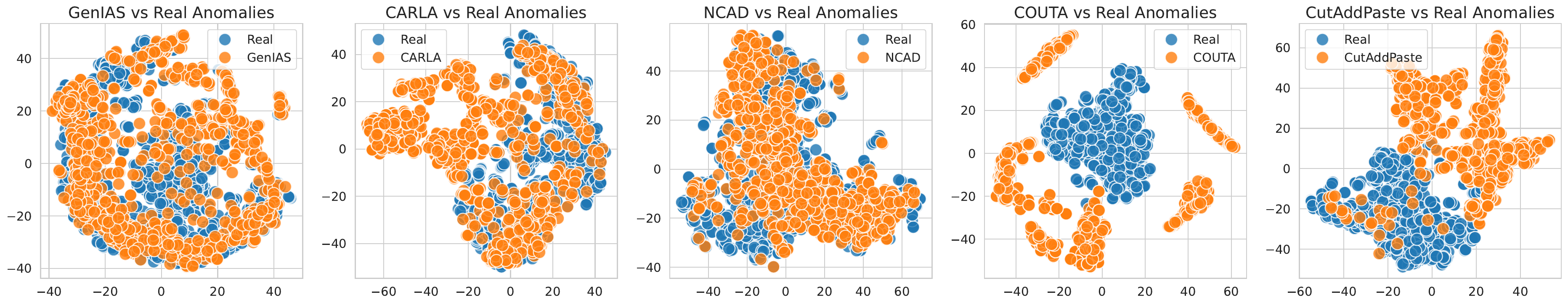}
        \label{fig:kpi_tsne}
        \vspace{-15pt}
        \caption{KPI - Entity c69a50cf-ee03-3bd7-831e-407d36c7ee91}
        \vspace{5pt}
    \end{subfigure}
    \par
    \begin{subfigure}{0.8\textwidth}
        \includegraphics[width=\textwidth]{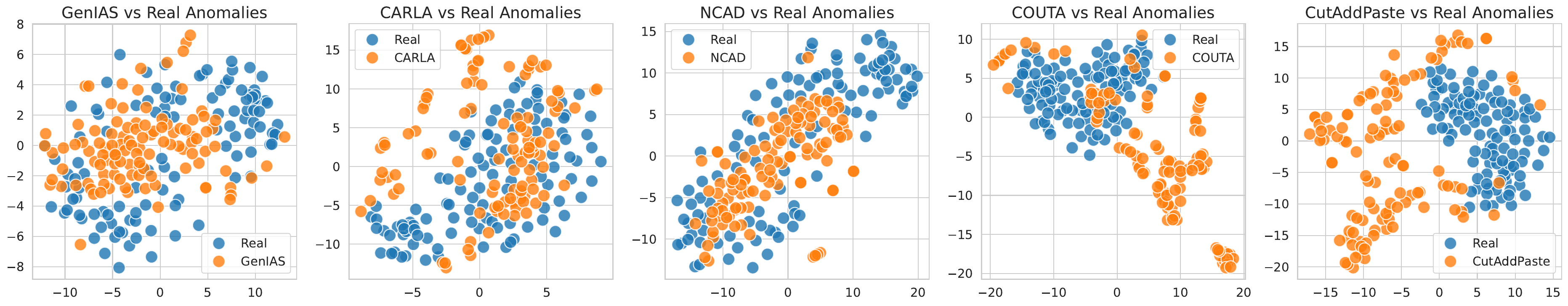}
        \label{fig:smd_tsne}
        \vspace{-15pt}
        \caption{SMD - Entity machine-2-3}
        \vspace{5pt}
    \end{subfigure}
    \par
    \begin{subfigure}{0.8\textwidth}
        \includegraphics[width=\textwidth]{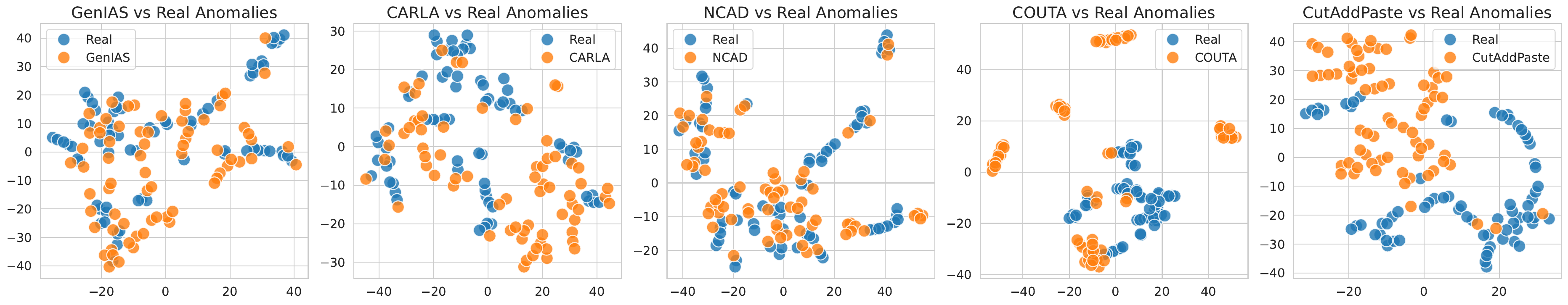}
        \label{fig:ucr_tsne}
        \vspace{-15pt}
        \caption{UCR - Entity 222\_UCR\_Anomaly\_mit14046longtermecg\_56123\_91200\_91700}
        \vspace{5pt}
    \end{subfigure}
    \vspace{-5pt}
    \caption{Each plot visualizes a different time series (TS) injection or generation method using t-SNE. Blue markers represent the original anomalous TS, while orange markers represent the injected or generated anomalous TS.}
    \vspace{-5pt}
    \label{fig:ext_tsne}
\end{figure*}

\subsubsection{Additional Visualizations}
\label{sec:apd:vis}
To further illustrate the realistic anomaly generation of GenIAS, we present three additional visualizations comparing generated and real anomalies from both GenIAS and the US-i baselines in Figure \ref{fig:ext_tsne}. The observations are consistent with Section~\ref{sec:visual} in the main text, demonstrating that GenIAS produces comparatively more realistic anomalies with better smoothness in the latent space.

\end{document}